\documentclass[11pt]{article}
\usepackage[margin=1in]{geometry}
\usepackage{amsmath,amssymb,amsthm,mathtools}
\usepackage[round]{natbib}
\usepackage{graphicx}
\usepackage{booktabs}
\usepackage{xcolor} 

\usepackage{comment}
\usepackage[ruled]{algorithm2e}
\usepackage{subcaption}

\DeclareCaptionLabelFormat{tableseries}{Table~#2}
\usepackage[normalem]{ulem}

\newcommand{\Dir}{\mathrm{Dir}}
\newcommand{\E}{\mathbb{E}}
\newcommand{\Var}{\mathrm{Var}}

\newcommand{\PE}{\mathrm{PE}}

\theoremstyle{plain} 
\newtheorem{theorem}{Theorem}
\newtheorem{lemma}[theorem]{Lemma}           
\newtheorem{proposition}[theorem]{Proposition}
\newtheorem{corollary}[theorem]{Corollary}

\theoremstyle{definition} 
\newtheorem{definition}[theorem]{Definition}

\theoremstyle{remark} 

\usepackage{microtype}
\usepackage[hidelinks]{hyperref}

\title{Improving the Predictive Performance of Bootstrap Aggregating by Dirichlet Resampling}
\author{
  Quoc Viet Le\thanks{Department of Statistics, University of Wisconsin--Madison. Email: \texttt{qle7@wisc.edu}}
  \and
  Joonha Park\thanks{Department of Mathematics, University of Kansas. Email: \texttt{j.park@ku.edu}}
}
\date{}

\begin{document}
\maketitle

\begin{abstract}
  We revisit Breiman's observation that reducing inter-tree correlation without weakening individual trees can improve random forests. Building on this principle, we introduce two variants:
Dirichlet-Multinomial Bagging Random Forest (DM) and Dirichlet-Weighted Random Forest (DW).
Both modulate sample reweighting via a concentration parameter $\alpha>0$. We provide a simple theoretical criterion that clarifies when these variants behave indistinguishably from standard random forests, and we use it to guide a lightweight tuning strategy. In a controlled evaluation on public classification benchmarks, DM and DW are consistently competitive and often stronger than other random-forest (RF) baselines, with negligible additional runtime.
\end{abstract}

\section{Introduction}

Random Forests (RF) attain strong predictive performance via randomized trees and bootstrap aggregating (bagging). Theorem~2.3 of \cite{Breiman2001RF} gives an upper bound on the asymptotic generalized error $\PE^*$ for Random Forests in terms of two quantities: the strength of individual tree $s$ and mean inter-tree correlation $\rho$, which is given as
\[ \PE^\ast \;\le\; \rho\,\frac{1-s^2}{s^2}.\]

This viewpoint motivates mechanisms that reduce $\rho$ to lower the generalized error upper bound. We study two closely related Dirichlet resampling mechanisms that target the data-side randomness of RFs: (i) a \emph{Dirichlet-Multinomial Bagging Random Forest} (DM) that draws each tree’s bag via a Dirichlet–multinomial; and (ii) a \emph{Dirichlet-Weighted Random Forest} (DW) that keeps all $n$ training points but assigns Dirichlet weights.
A single hyperparameter $\alpha$ controls the overlap between pairs of bootstrap samples (for DM) or the concentration of  weights (for DW).

\paragraph{Contributions.} To the best of our knowledge, no prior framework controls bootstrap aggregating via Dirichlet-based resampling with a single parameter. We introduce Dirichlet-Multinomial Bagging RF (DM) and Dirichlet-Weighted RF (DW)---two single-parameter ($\alpha>0$) random-forest variants that modulate sample resampling/weighting while holding model capacity fixed. We design a finite-sample \(\alpha\)-cap to control how close our methods are to RF at resolution \(\varepsilon\). We prove a sufficient condition under which DM/DW is (up to a tolerance \(\varepsilon\)) indistinguishable from standard RF, yielding an \(n\)-scaled upper bound that suggests a principled \(\alpha\)-sweep. Complete proofs of our theoretical results are deferred to Appendix~\ref{app:proofs}, and additional experimental results are reported in Appendix~\ref{app:additional-results}.

\section{Related Work}

Bagging was introduced by Breiman \cite{Breiman1996Bagging} as a way to stabilize unstable learners by training on multiple bootstrap replicates of the data; the bootstrap itself goes back to \cite{Efron1979Bootstrap} (see \cite{HastieTibshiraniFriedman2009ESL} for a textbook treatment of bagging and ensemble methods). In his technical note on ``random features'' and the formal Random Forests (RF) paper, Breiman developed the strength--correlation view and the generalization-error bound that motivates decreasing inter-tree correlation $\rho$ while maintaining strength $s$ \cite{Breiman1999RFNote,Breiman2001RF}. Empirical work has explicitly linked these two quantities to RF error \cite{Bernard2010StrengthCorrelation}.
\cite{Buhlmann2002AnalyzingBagging} analyzed the variance-reduction effect of bagging and proposed subagging (sampling without replacement) as a principled alternative to the Efron bootstrap. For inference, \cite{MentchHooker2016JMLR} showed that RFs built with subsampling can be cast as $U$-statistics, yielding asymptotic normality for predictions; subsequent developments leveraged ``honest'' forests for causal inference \cite{WagerAthey2018CausalForest}. On consistency, Scornet, Biau, and Vert \cite{Scornet2015Consistency} established convergence results for original RFs within an additive-function setting.
To handle streaming settings, \cite{Oza2005Online} replaced multinomial bootstrap counts with independent $N_i\sim \mathrm{Poisson}(1)$ per observation, an online approximation widely adopted in practice. Reducing inter-tree correlation can also be achieved by randomizing the feature space as Random Subspace by \cite{Ho1998Subspace}, by randomizing split thresholds more aggressively as Extremely Randomized Trees by \cite{Geurts2006ExtraTrees}, or by jointly randomizing both samples and features as Random Patches by \cite{Louppe2012RandomPatches}. \cite{Rubin1981BB}’s Bayesian bootstrap replaces resampling by Dirichlet weights on the empirical support ; \cite{NewtonRaftery1994WLB} introduced the weighted likelihood bootstrap which generalizes this to approximate Bayes draws.
For trees and ensembles, Bayesian-bootstrap weighting has been explored by \cite{ClydeLee2001BBB,Lee2004LosslessBayesBag}, and fully/proper Bayesian variants frame forests as posterior draws by \cite{Taddy2015BayesForests,Galvani2021PBBTrees}.

\section{Method}
\subsection{Dirichlet-Multinomial Bagging Random Forest (DM)}

Given the training data $\{(x_i,y_i)\}_{i=1}^n$, Dirichlet–Multinomial (DM) trees are generated as follows.
For each tree index $b\in\{1,\dots,B\}$:
\begin{enumerate}
\item Draw a probability vector over training indices
\[
\mathbf p^{(b)} \sim \mathrm{Dir}\!\big(\alpha\,\mathbf{1}_n\big),
\]
where $\alpha>0$ is the concentration parameter and $\mathbf{1}_n\in\mathbb{R}^n$ is the all-ones vector.
\item Form a size-$n$ bag by sampling with replacement the training data with counts
\[
K^{(b)} \sim \text{Multinomial }\!\big(n,\,\mathbf p^{(b)}\big),
\]
and train a CART learner on the sample.
\end{enumerate}

\subsection{Dirichlet-Weighted Random Forest (DW)}
For each tree index $b=1,\dots,B$, draw a symmetric Dirichlet weight vector
$\mathbf w^{(b)} \sim \Dir(\alpha\,\mathbf 1_n)$ over the $n$ training points and fit a CART using the weights $\mathbf w^{(b)}$.

This method retains the entire training set as the empirical support for every tree (all $n$ points receive positive weight almost surely), so there is no natural notion of OOB samples for this variant.

\subsection{Relation to Existing Methods}

The Bayesian bootstrap (BB) draws $\mathbf w\sim\Dir(1,\ldots,1)$ over the empirical support,
while the Weighted Likelihood Bootstrap (WLB) draws $\tilde w_i \stackrel{iid}{\sim}\mathrm{Exp}(1)$
and uses normalized weights $\tilde{\mathbf w}/\sum_i \tilde w_i$.
Since $\mathrm{Exp}(1)=\mathrm{Gamma}(1,1)$ and a normalized independent Gamma vector is Dirichlet,
both BB and WLB induce exactly the same weighting distribution as our Dirichlet-weighted scheme
with $\alpha=1$. Dirichlet–weighted Random Forest therefore strictly generalizes these schemes: it recovers the Bayesian/weighted-likelihood bootstrap Random Forest at $\alpha=1$, while $\alpha$ continuously controls weight concentration (Refer to Appendix~\ref{app:bbwlb}).

By Proposition~\ref{prop:whp-markov}, in large samples, the Bayesian Random Forests or weighted-likelihood bootstrap Random Forests are, with high probability, almost indistinguishable from the standard Random Forest without bootstrapping.
Accordingly, using $\alpha>1$ offers little practical value.

\section{Theoretical Analysis}
\begin{definition}[$\varepsilon$–uniform weights in $\ell_p$]
\label{def:eps-uniform}
Let $\mathbf u=\tfrac{1}{n}\mathbf{1}_n$ be the uniform weights of sample of size $n$ and
$v$ be a vector of probabilities or normalized weights, we say $\mathbf v$ is $\varepsilon$–uniform weights in $\ell_p$ if $\|\mathbf v - \mathbf u\|_p \le \varepsilon$.
\end{definition}

As $\alpha \to \infty$, $\mathbf{p}^{(b)} \to \mathbf{u}$ in probability, and the DM mechanism recovers the standard bootstrap used in Efron's Random Forests.
For DW, the fact that $\mathbf{w}^{(b)} \to \mathbf{u}$ in probability as $\alpha \to \infty$ implies that it approaches a Random Forest without bootstrapping, since all points receive equal weight in every tree.

\begin{lemma}[Second moment]
Let $\mathbf p\sim\mathrm{Dir}(\alpha\,\mathbf 1_n)$ and $\mathbf u=\frac{1}{n}\mathbf 1_n$. Then
\[
\mathbb{E}\|\mathbf p - \mathbf u\|_2^2
= \frac{n-1}{\,n\,(n\alpha+1)\,}
\]
\end{lemma}

\begin{proposition}[High-probability bound]
\label{prop:whp-markov}
Fix $\varepsilon>0$ and confidence level $1-\delta\in(0,1)$. If
\[
\alpha \;\ge\; \frac{\,n-1\,}{\,n^2\,\varepsilon^2\,\delta\,}\;-\;\frac{1}{n}\,,
\]
then a DM draw $\mathbf p\sim\mathrm{Dir}(\alpha\,\mathbf 1_n)$ is $\epsilon$-uniform in the $\ell_2$ sense with probability at least $1-\delta$.
\end{proposition}

\textit{Proof.}
By Markov's inequality applied to the nonnegative random variable $\|\mathbf p - \mathbf u\|_2^2$,
\[
\Pr\!\big(\|\mathbf p - \mathbf u\|_2 > \varepsilon\big)
\le \frac{\mathbb{E}\|\mathbf p - \mathbf u\|_2^2}{\varepsilon^2}
= \frac{\,n-1\,}{\,n\,(n\alpha+1)\,\varepsilon^2}\,.
\]
Requiring the right-hand side to be at most $\delta$ and solving for $\alpha$ yields the stated bound.

\begin{corollary}[Expectation bound]
\label{cor4}
If
\[
\alpha \;\ge\; \frac{\,n-1\,}{\,n^2\,\varepsilon^2\,}\;-\;\frac{1}{n}\,, \text{ then } \mathbb{E}\|\mathbf p - \mathbf u\|_2 \le \varepsilon.
\]
\end{corollary}

\noindent\textit{Remark (OOB Analysis).}
Let $U$ be the number of distinct training data points included in a bag. For $\alpha>0$,
\begin{align*}
p(\alpha):= \E\!\left[\frac{U}{n}\right]= 1 - \frac{B\!\left(\alpha,\,(n-1)\alpha + n\right)}{B\!\left(\alpha,\,(n-1)\alpha\right)},
\end{align*}
where $B(\cdot,\cdot)$ denotes the Beta function. The OOB (out-of-bag) rate equals $1-p(\alpha)$. A detailed derivation is provided in Appendix~\ref{app:p-alpha-derivation}.

Let $\textbf{p}\!\sim\!\Dir(\alpha\,\mathbf 1_n)$ and $\textbf{K$|$p}\!\sim\!\text{Mult}(n,\mathbf p)$ be the
sampling counts under Dirichlet-Multinomial (DM) resampling.
By Corollary~\ref{cor4}, it is useful to define the dataset–dependent threshold
\[
\alpha_{\mathrm{RF}}(n,\varepsilon)\;:=\;\frac{n-1}{n^2\varepsilon^2}-\frac{1}{n}\,.
\]
For $\alpha\ge \alpha_{\mathrm{RF}}(n,\varepsilon)$, the Dirichlet draw $\textbf{p}$ is
$\varepsilon$–uniform in the $\ell_{2}$ sense in expectation. 
Hence, the induced DM counts
$K$ are indistinguishable from Efron’s bootstrap
$\text{Mult}\!\big(n,\mathbf u\big)$ at resolution $\varepsilon$. 
Thus $\alpha_{\mathrm{RF}}(n,\varepsilon)$ marks a saturation point, after which further increases have little effect.

\section{Experiments}
We set $\alpha$ values in DM/DW implementations as follows.
Set 
\(
\alpha_{\max}=\alpha_{\mathrm{RF}}(n,\varepsilon),
\)
with $\varepsilon = 0.01$ for all datasets.
We sweep $\alpha$ over $(0,\alpha_{\max}]$. 
The cap $\alpha_{\max}$ realizes the DM$\to$RF (bootstrap) limit at resolution $\varepsilon$.
For DW, we apply the same $\alpha$-sweep. In this case, using $\alpha > \alpha_{\mathrm{RF}}(n,\varepsilon)$ produces $\varepsilon$-uniform weights, which approximately corresponds to a no-bootstrap regime.
 
Algorithm~\ref{alg:alpha-sampler-simple} summarizes a random search design for the Dirichlet concentration parameter $\alpha$, following the random-search philosophy of \cite{10.5555/2188385.2188395}.
A brief justification is provided in Appendix~\ref{app:alpha_search}.

\begin{algorithm}[t]
\caption{Random-$\alpha$ search}
\label{alg:alpha-sampler-simple}
\DontPrintSemicolon
\SetKwInOut{Input}{Input}\SetKwInOut{Output}{Output}
\Input{training data size $n$; $\alpha$-draws $M$; resolution $\varepsilon$; search bounds $\alpha_{\min}^{\rm ov},\alpha_{\max}^{\rm ov}$ (optional)}
\Output{a grid $\mathcal{A}$ of $\alpha$ values}

Let $\alpha_{\max}\leftarrow \alpha_{\mathrm{RF}}(n,\varepsilon)$ or $\alpha_{\max}^{\rm ov}$ (if provided)\;
Let $\alpha_{\min}\leftarrow 0$ or $\alpha_{\min}^{\rm ov}$ (if provided)\;

Let $z_{\min}\leftarrow \log(1+n\alpha_{\min})$ and  $z_{\max}\leftarrow \log(1+n\alpha_{\max})$\;
Draw $U_1,\dots,U_M \stackrel{iid}{\sim}U(0,1)$\;
\For{$m=1$ \KwTo $M$}{
  $z_m \leftarrow z_{\min}+U_m\,(z_{\max}-z_{\min})$\;
  $\alpha_m \leftarrow \dfrac{e^{z_m}-1}{n}$\;
}
\Return $\mathcal{A} = \mathrm{sort}\big(\{\alpha_m\}_{m=1}^M\big)$\;
\end{algorithm}

We use fifteen benchmark datasets from the UCI and OpenML \cite{vanschoren2014openml} repositories.
Details of the datasets are provided in Appendix~\ref{app:dataset-info}.
For each dataset, class imbalance is categorized as low ( majority class $\lesssim 60 \%$), medium ( $\approx 60-80 \%$), or high ( $\gtrsim 80 \%$).

\subsection{Experimental Settings}
The baselines and Dirichlet-Multinomial (DM) and Dirichlet Weighting (DW) Random Forests share the same base learner and hyperparameters. The base learner is an unpruned CART decision tree \cite{breiman1984cart} (classification with Gini impurity), implemented via the scikit-learn \cite{pedregosa2011scikit} decision-tree classes and mirrored by our DM/DW implementations that follow the scikit-learn API. We set the number of trees to 200 ; number of features considered for splitting nodes to $m_{try}=\sqrt{d}$ for classification tasks (configured through \texttt{max\_features} and the corresponding setting in our classes); maximum depth is unlimited; and the minimum number of samples per leaf is 1.

Each method was run 10 times on each dataset.
For consistency, all methods used the same 10 train/test splits per dataset with an 80/20 ratio.
Fixed random seeds were employed to ensure reproducibility.

\subsubsection{Baselines}
We evaluated four benchmark methods, configured as follows.
\begin{itemize}
\raggedright
\item \textbf{RF-bootstrap} (\texttt{RandomForestClassifier}): \texttt{n\_estimators=200}, \texttt{bootstrap=True}, \texttt{criterion="gini"}, \texttt{max\_features="sqrt"}.
\item \textbf{RF-no-bootstrap} (\texttt{RandomForestClassifier}): same as above except \texttt{bootstrap=False}.
\item \textbf{ET} (\texttt{ExtraTreesClassifier}): \texttt{n\_estimators=200}, \texttt{bootstrap=False}, \texttt{criterion="gini"}, \texttt{max\_features="sqrt"}.
\item \textbf{Subsample bagging} (bagging over \texttt{DecisionTreeClassifier}): base estimator \texttt{DecisionTreeClassifier(max\_features="sqrt" random\_state=seed)}; \texttt{n\_estimators=200}; without-replacement subsampling with \(\texttt{max\_samples}=\mathrm{round}(0.632 \times \text{train\_size})\); \texttt{bootstrap=False}.
\end{itemize}

\subsubsection{Dirichlet-Multinomial (DM) / Dirichlet Weighting (DW)}
We developed two Python functions implementing the DM/DW Random Forests.
\begin{itemize}
\raggedright
\item \textbf{DirichletMultinomialBaggingRFClassifier} (Algorithm~\ref{alg:dmrf})
\item \textbf{DirichletWeightedRFClassifier} (Algorithm~\ref{alg:dwrf})
\end{itemize}
We compared these methods with the baselines under identical settings (\texttt{n\_estimators=200}, \texttt{max\_features="sqrt"}).
The concentration parameter \(\alpha\) was varied via a random search (Algorithm~\ref{alg:alpha-sampler-simple}) and passed to the respective function as \texttt{alpha}.

\begin{algorithm}[t]
\caption{Dirichlet–Multinomial Bagging Random Forest (DM-RF)}
\label{alg:dmrf}
\DontPrintSemicolon
\SetKwInOut{Input}{Input}\SetKwInOut{Output}{Output}

\Input{Training data $\{(x_i,y_i)\}_{i=1}^n$; number of trees $B$; concentration parameter $\alpha>0$; number of features considered per node split $m_{try}$.}
\Output{Ensemble $\mathcal{F}=\{T^{(b)}\}_{b=1}^B$.}

\BlankLine
\For{$b \gets 1$ \KwTo $B$}{
    Draw probabilities $\mathbf p^{(b)} \sim \Dir(\alpha\,\mathbf 1_n)$\;

    Draw counts $K^{(b)} \sim \text{Mult}\!\big(n,\,\mathbf p^{(b)}\big)$

    Form a bootstrap dataset $\mathcal{D}^{(b)}$ by assigning the multiplicity $K^{(b)}_i$ to the $i$-th data point \;

    Train a CART $T^{(b)}$ on $\mathcal{D}^{(b)}$ using $m_{try}$ randomly selected features at each split\;

    Store $T^{(b)}$ in $\mathcal{F}$\;
}
\end{algorithm}

\begin{algorithm}[t]
\caption{Dirichlet–Weighted Random Forest (DW-RF)}
\label{alg:dwrf}
\DontPrintSemicolon
\SetKwInOut{Input}{Input}\SetKwInOut{Output}{Output}
\Input{Training data $\{(x_i,y_i)\}_{i=1}^n$; number of trees $B$; concentration parameter $\alpha>0$; number of features considered per node split $m_{try}$.}
\Output{Ensemble $\mathcal{F}=\{T^{(b)}\}_{b=1}^B$.}
\BlankLine
\For{$b \gets 1$ \KwTo $B$}{
  Draw sample weights $\mathbf w^{(b)} \sim \Dir(\alpha\,\mathbf 1_n)$\;
  Train a CART $T^{(b)}$ on the full dataset using \texttt{sample\_weight} $=\mathbf w^{(b)}$ and $m_{try}$ randomly selected features at each split\;
  Store $T^{(b)}$ in $\mathcal{F}$\;
}
\end{algorithm}

\subsection{Results}
For each replication, the 200 tree predictions were aggregated using soft voting: the predicted label for a test point $x$ was given by
$\hat{y}(x)=\arg\max_k \hat{p}_k(x)$, where $\hat{p}_k(x)=\frac{1}{B}\sum_{b=1}^B \hat{p}_k^{(b)}(x)$.
We evaluated predictive performance using three metrics: classification accuracy, log-loss, and the area under the ROC curve (AUROC). The results below summarize the experiments in which DM/DW outperformed the baseline at the 90\% significance level (\(z > 1.64\)).




\begin{table}[p]
\centering
\caption{Accuracy (Mean $\pm$ SE) of DM/DW vs four baselines (ET, RF-bootstrap, RF-no-bootstrap, Subsample). Superscripts for DM/DW indicate: (1) baselines significantly outperformed (10\% significance level), and (2) baselines for which DM/DW is not significantly worse.
}
\label{tab:acc_superscripts_no_brf_corrected}
\vspace{\baselineskip}

\scriptsize
\setlength{\tabcolsep}{4pt}\renewcommand{\arraystretch}{1.05}
\begin{tabular}{lcccccc}
\textbf{Dataset} & \textbf{ET} & \textbf{RF-bootstrap} & \textbf{RF-no-bootstrap} & \textbf{Subsample} & \textbf{DM} & \textbf{DW}\\
\midrule
adult & 0.8293$\pm$0.0009 & 0.8491$\pm$0.0011 & 0.8431$\pm$0.0010 & 0.8499$\pm$0.0011 & 0.8561$\pm$0.0011$^{4,4}$ & 0.8567$\pm$0.0011$^{4,4}$ \\
banknote & 0.9993$\pm$0.0005 & 0.9938$\pm$0.0012 & 0.9945$\pm$0.0011 & 0.9942$\pm$0.0006 & 0.9945$\pm$0.0006$^{0,3}$ & 0.9971$\pm$0.0009$^{3,3}$ \\
breast\_cancer & 0.9649$\pm$0.0035 & 0.9649$\pm$0.0039 & 0.9667$\pm$0.0025 & 0.9684$\pm$0.0035 & 0.9684$\pm$0.0035$^{0,4}$ & 0.9702$\pm$0.0033$^{0,4}$ \\
credit\_g & 0.7510$\pm$0.0084 & 0.7650$\pm$0.0098 & 0.7540$\pm$0.0078 & 0.7615$\pm$0.0080 & 0.7710$\pm$0.0077$^{1,4}$ & 0.7615$\pm$0.0093$^{0,4}$ \\
digits & 0.9839$\pm$0.0015 & 0.9769$\pm$0.0013 & 0.9781$\pm$0.0017 & 0.9767$\pm$0.0022 & 0.9781$\pm$0.0017$^{0,3}$ & 0.9817$\pm$0.0012$^{3,4}$ \\
ecoli & 0.8809$\pm$0.0071 & 0.8941$\pm$0.0098 & 0.8824$\pm$0.0076 & 0.8985$\pm$0.0080 & 0.9000$\pm$0.0092$^{0,4}$ & 0.8971$\pm$0.0076$^{0,4}$ \\
glass & 0.7907$\pm$0.0125 & 0.7767$\pm$0.0131 & 0.7698$\pm$0.0203 & 0.7837$\pm$0.0155 & 0.7791$\pm$0.0144$^{0,4}$ & 0.7837$\pm$0.0211$^{0,4}$ \\
ionosphere & 0.9507$\pm$0.0057 & 0.9423$\pm$0.0065 & 0.9380$\pm$0.0079 & 0.9451$\pm$0.0074 & 0.9465$\pm$0.0062$^{0,4}$ & 0.9507$\pm$0.0067$^{0,4}$ \\
iris & 0.9600$\pm$0.0097 & 0.9500$\pm$0.0124 & 0.9433$\pm$0.0071 & 0.9500$\pm$0.0124 & 0.9633$\pm$0.0126$^{0,4}$ & 0.9533$\pm$0.0089$^{0,4}$ \\
letter & 0.9723$\pm$0.0008 & 0.9644$\pm$0.0009 & 0.9679$\pm$0.0009 & 0.9644$\pm$0.0009 & 0.9602$\pm$0.0012$^{0,0}$ & 0.9682$\pm$0.0009$^{2,3}$ \\
magic04 & 0.8807$\pm$0.0020 & 0.8835$\pm$0.0016 & 0.8831$\pm$0.0015 & 0.8835$\pm$0.0018 & 0.8838$\pm$0.0015$^{0,4}$ & 0.8849$\pm$0.0015$^{1,4}$ \\
optdigits & 0.9861$\pm$0.0010 & 0.9844$\pm$0.0010 & 0.9861$\pm$0.0010 & 0.9846$\pm$0.0011 & 0.9825$\pm$0.0014$^{0,2}$ & 0.9859$\pm$0.0011$^{0,4}$ \\
page\_blocks & 0.9709$\pm$0.0012 & 0.9728$\pm$0.0012 & 0.9692$\pm$0.0014 & 0.9722$\pm$0.0011 & 0.9699$\pm$0.0015$^{0,4}$ & 0.9733$\pm$0.0013$^{1,4}$ \\
pendigits & 0.9935$\pm$0.0003 & 0.9918$\pm$0.0005 & 0.9928$\pm$0.0005 & 0.9918$\pm$0.0004 & 0.9904$\pm$0.0006$^{0,0}$ & 0.9934$\pm$0.0005$^{2,4}$ \\
phoneme & 0.9182$\pm$0.0022 & 0.9127$\pm$0.0022 & 0.9131$\pm$0.0021 & 0.9119$\pm$0.0026 & 0.9014$\pm$0.0022$^{0,0}$ & 0.9191$\pm$0.0024$^{3,4}$ \\
\end{tabular}
\vspace{0.5em}

\end{table}

\vspace{\baselineskip}

\begin{table}[p]
\centering
\caption{Log loss (Mean$\pm$SE) of DM/DW vs four baselines}
\label{tab:logloss_superscripts}
\vspace{\baselineskip}

\scriptsize
\setlength{\tabcolsep}{4pt}\renewcommand{\arraystretch}{1.05}
\centering

\begin{tabular}{lcccccc}
\textbf{Dataset} & \textbf{ET} & \textbf{RF-bootstrap} & \textbf{RF-no-bootstrap} & \textbf{Subsample} & \textbf{DM} & \textbf{DW} \\
\midrule
adult & 0.5109$\pm$0.0078 & 0.3573$\pm$0.0034 & 0.4893$\pm$0.0069 & 0.3541$\pm$0.0039 & 0.3332$\pm$0.0044$^{4,4}$ & 0.3321$\pm$0.0043$^{4,4}$ \\
banknote & 0.0106$\pm$0.0028 & 0.0243$\pm$0.0064 & 0.0231$\pm$0.0066 & 0.0205$\pm$0.0048 & 0.0199$\pm$0.0047$^{0,3}$ & 0.0149$\pm$0.0049$^{0,4}$ \\
breast\_cancer & 0.1195$\pm$0.0182 & 0.1309$\pm$0.0197 & 0.1168$\pm$0.0155 & 0.1186$\pm$0.0170 & 0.1203$\pm$0.0165$^{0,4}$ & 0.1183$\pm$0.0166$^{0,4}$ \\
credit\_g & 0.5874$\pm$0.0410 & 0.5664$\pm$0.0395 & 0.5853$\pm$0.0380 & 0.5738$\pm$0.0391 & 0.5601$\pm$0.0387$^{0,4}$ & 0.5724$\pm$0.0402$^{0,4}$ \\
digits & 0.0792$\pm$0.0045 & 0.0978$\pm$0.0054 & 0.0906$\pm$0.0046 & 0.0987$\pm$0.0066 & 0.0904$\pm$0.0051$^{0,3}$ & 0.0857$\pm$0.0048$^{1,4}$ \\
ecoli & 0.3245$\pm$0.0269 & 0.3222$\pm$0.0319 & 0.3307$\pm$0.0262 & 0.3161$\pm$0.0288 & 0.3100$\pm$0.0301$^{0,4}$ & 0.3135$\pm$0.0261$^{0,4}$ \\
glass & 0.8735$\pm$0.0833 & 0.9864$\pm$0.1162 & 1.0522$\pm$0.1531 & 0.9371$\pm$0.0982 & 0.9521$\pm$0.1080$^{0,4}$ & 0.9400$\pm$0.1064$^{0,4}$ \\
ionosphere & 0.2007$\pm$0.0247 & 0.2532$\pm$0.0297 & 0.2644$\pm$0.0380 & 0.2309$\pm$0.0334 & 0.2211$\pm$0.0296$^{0,4}$ & 0.2078$\pm$0.0299$^{0,4}$ \\
iris & 0.1492$\pm$0.0302 & 0.1771$\pm$0.0399 & 0.1943$\pm$0.0269 & 0.1776$\pm$0.0405 & 0.1433$\pm$0.0307$^{0,4}$ & 0.1675$\pm$0.0269$^{0,4}$ \\
letter & 0.1068$\pm$0.0010 & 0.1429$\pm$0.0021 & 0.1295$\pm$0.0018 & 0.1427$\pm$0.0020 & 0.1526$\pm$0.0029$^{0,0}$ & 0.1281$\pm$0.0019$^{2,3}$ \\
magic04 & 0.2952$\pm$0.0033 & 0.2827$\pm$0.0033 & 0.2835$\pm$0.0032 & 0.2827$\pm$0.0036 & 0.2816$\pm$0.0034$^{1,4}$ & 0.2797$\pm$0.0034$^{1,4}$ \\
optdigits & 0.0646$\pm$0.0025 & 0.0739$\pm$0.0030 & 0.0651$\pm$0.0026 & 0.0735$\pm$0.0034 & 0.0790$\pm$0.0037$^{0,2}$ & 0.0681$\pm$0.0028$^{0,4}$ \\
page\_blocks & 0.1127$\pm$0.0023 & 0.1096$\pm$0.0022 & 0.1165$\pm$0.0026 & 0.1107$\pm$0.0022 & 0.1110$\pm$0.0026$^{0,4}$ & 0.1087$\pm$0.0023$^{1,4}$ \\
pendigits & 0.0221$\pm$0.0007 & 0.0307$\pm$0.0010 & 0.0250$\pm$0.0009 & 0.0309$\pm$0.0010 & 0.0341$\pm$0.0012$^{0,0}$ & 0.0240$\pm$0.0009$^{2,3}$ \\
phoneme & 0.2752$\pm$0.0046 & 0.2913$\pm$0.0054 & 0.2887$\pm$0.0052 & 0.2954$\pm$0.0061 & 0.3056$\pm$0.0060$^{0,1}$ & 0.2716$\pm$0.0059$^{3,4}$ \\

\end{tabular}
\vspace{0.25em}
\end{table}

\vspace{\baselineskip}

\begin{table}[p]
\centering

\caption{AUROC (Mean$\pm$SE) of DM/DW vs four baselines.}
\label{tab:auroc}
\vspace{\baselineskip}

\scriptsize
\setlength{\tabcolsep}{4pt}\renewcommand{\arraystretch}{1.05}
  \begin{tabular}{lcccccc}
\textbf{Dataset} & \textbf{ET} & \textbf{RF-bootstrap} & \textbf{RF-no-bootstrap} & \textbf{Subsample} & \textbf{DM} & \textbf{DW} \\
\midrule
adult & 0.8779$\pm$0.0010 \ & 0.9018$\pm$0.0010 \ & 0.8870$\pm$0.0011 \ & 0.9023$\pm$0.0011 \ & 0.9053$\pm$0.0011$^{4,4}$  \ & 0.9055$\pm$0.0011$^{4,4}$  \ \\
banknote 
& 0.9992$\pm$0.0004 \ & 0.9965$\pm$0.0013 \ & 0.9970$\pm$0.0012 \ & 0.9972$\pm$0.0006 \ & 0.9972$\pm$0.0006$^{0,3}$  \ & 0.9985$\pm$0.0008$^{0,4}$  \ \\
breast\_cancer 
& 0.9936$\pm$0.0027 \ & 0.9928$\pm$0.0031 \ & 0.9938$\pm$0.0022 \ & 0.9943$\pm$0.0028 \ & 0.9942$\pm$0.0028$^{0,4}$  \ & 0.9944$\pm$0.0026$^{0,4}$  \ \\
credit\_g 
& 0.8446$\pm$0.0145 \ & 0.8537$\pm$0.0142 \ & 0.8450$\pm$0.0129 \ & 0.8498$\pm$0.0137 \ & 0.8577$\pm$0.0138$^{0,4}$  \ & 0.8495$\pm$0.0145$^{0,4}$  \ \\
digits 
& 0.9984$\pm$0.0003 \ & 0.9974$\pm$0.0005 \ & 0.9977$\pm$0.0004 \ & 0.9973$\pm$0.0007 \ & 0.9977$\pm$0.0005$^{0,4}$  \ & 0.9980$\pm$0.0004$^{0,4}$  \ \\
glass 
& 0.8070$\pm$0.0409 \ & 0.7848$\pm$0.0476 \ & 0.7706$\pm$0.0598 \ & 0.7955$\pm$0.0477 \ & 0.7924$\pm$0.0516$^{0,4}$  \ & 0.7961$\pm$0.0515$^{0,4}$ \ \\
ionosphere 
& 0.9917$\pm$0.0019 \ & 0.9828$\pm$0.0036 \ & 0.9754$\pm$0.0055 \ & 0.9870$\pm$0.0039 \ & 0.9888$\pm$0.0031$^{1,4}$  \ & 0.9907$\pm$0.0029$^{2,4}$  \ \\
iris 
& 0.9963$\pm$0.0045 \ & 0.9937$\pm$0.0069 \ & 0.9910$\pm$0.0054 \ & 0.9937$\pm$0.0069 \ & 0.9965$\pm$0.0047$^{0,4}$  \ & 0.9949$\pm$0.0049$^{0,4}$  \ \\
letter 
& 0.9984$\pm$0.0001 \ & 0.9961$\pm$0.0003 \ & 0.9970$\pm$0.0003 \ & 0.9961$\pm$0.0003 \ & 0.9953$\pm$0.0005$^{0,2}$  \ & 0.9972$\pm$0.0003$^{2,3}$  \ \\
magic04 
& 0.9571$\pm$0.0009 \ & 0.9596$\pm$0.0009 \ & 0.9594$\pm$0.0008 \ & 0.9597$\pm$0.0010 \ & 0.9601$\pm$0.0010$^{1,4}$  \ & 0.9607$\pm$0.0010$^{1,4}$  \ \\
optdigits 
& 0.9993$\pm$0.0001 \ & 0.9989$\pm$0.0002 \ & 0.9993$\pm$0.0001 \ & 0.9989$\pm$0.0002 \ & 0.9987$\pm$0.0002$^{0,2}$  \ & 0.9992$\pm$0.0002$^{0,4}$  \ \\
page\_blocks 
& 0.9967$\pm$0.0003 \ & 0.9969$\pm$0.0003 \ & 0.9964$\pm$0.0004 \ & 0.9968$\pm$0.0003 \ & 0.9967$\pm$0.0004$^{0,4}$  \ & 0.9970$\pm$0.0003$^{0,4}$  \ \\
pendigits 
& 0.9997$\pm$0.0001 \ & 0.9994$\pm$0.0001 \ & 0.9996$\pm$0.0001 \ & 0.9994$\pm$0.0001 \ & 0.9993$\pm$0.0002$^{0,3}$  \ & 0.9996$\pm$0.0001$^{0,4}$  \ \\
phoneme 
& 0.9625$\pm$0.0015 \ & 0.9591$\pm$0.0019 \ & 0.9598$\pm$0.0018 \ & 0.9585$\pm$0.0021 \ & 0.9565$\pm$0.0022$^{0,3}$  \ & 0.9631$\pm$0.0019$^{0,4}$  \ \\
  \end{tabular}
\vspace{\baselineskip}

\end{table}


Tables~\ref{tab:acc_superscripts_no_brf_corrected}--\ref{tab:auroc} report the three performance measures (accuracy, log-loss, and AUROC) with their standard errors.
For DM and DW, we report the best results over 10 $\alpha$ values for each dataset.
Additionally, we indicate two numbers for DM/DW : (1) the number of baselines significantly outperformed at the 10\% level, and (2) the number of baselines for which DM/DW is not significantly worse (i.e., at least as good) at the 10\% level.
Significance was assessed using two-tailed z-tests.

DW was strongly competitive, showing statistically significant outperformance over several baselines on most datasets, and performing at least as well as all baselines in almost all cases.
DM was comparable to all baselines on all but a few benchmarks, and it outperformed all four baselines on a single benchmark (the adult dataset) in terms of accuracy (see Appendix B for full experimental results).
We note that ET was frequently among the top-performing baselines.

Table~\ref{tab:alpha_vals} reports, for each dataset, the number of $\alpha$ values (out of 10) that yield performance comparable to the best.
On average, approximately 5.6 $\alpha$ values yielded similar performance, indicating that DM and DW are reasonably robust to the choice of $\alpha$, even when the sweep is conducted on a logarithmic scale.

\begin{table}[t]
\centering
\caption{The best $\alpha$ and the $\alpha$ values yielding comparable accuracy for DM/DW on each dataset. The training sample size $n$ is also shown. }
\label{tab:alpha_vals}
\vspace{\baselineskip}

\scriptsize
\setlength{\tabcolsep}{4pt}\renewcommand{\arraystretch}{1.05}
\begin{tabular}{lcccccc}

Dataset & Method & Best $\alpha$ & Comparable $\alpha$ Values & Count & $n$ \\
\midrule
adult & DM & 0.1573 & 0.0133, 0.1744 & 3 & 48842 \\
adult & DW & 0.0133 & 0.0066, 0.0091 & 3 & 48842 \\
banknote & DM & 4.8365 & 0.8498, 1.1743, 1.5397, 2.3157 & 5 & 1372 \\
banknote & DW & 0.0496 & 0.2862, 0.8498, 4.8365 & 4 & 1372 \\
breast\_cancer & DM & 5.7722 & 0.0570, 0.1428, 0.2972, 2.2322, 6.4756, 10.1280 & 7 & 569 \\
breast\_cancer & DW & 5.7722 & 0.0256, 0.0570, 0.1428, 0.2972, 2.2322, 6.4756, 10.1280 & 8 & 569 \\
credit\_g & DM & 2.4879 & 0.2544, 0.4834, 2.4058, 3.2354, 3.5596 & 6 & 1000 \\
credit\_g & DW & 3.2354 & 0.2544, 0.4834, 2.4058, 2.4879, 3.5596 & 6 & 1000 \\
digits & DM & 1.1196 & 0.5192, 1.0746, 4.8044 & 5 & 1797 \\
digits & DW & 0.5192 & 1.0746, 1.1196, 4.8044 & 5 & 1797 \\
ecoli & DM & 4.2772 & 0.1339, 0.1462, 0.2199, 2.1999, 2.7827, 17.1034, 18.8560, 25.8465 & 9 & 336 \\
ecoli & DW & 0.0236 & 0.1339, 0.1462, 0.2199, 2.1999, 2.7827, 4.2772, 25.8465 & 8 & 336 \\
glass & DM & 6.5920 & 3.5425, 9.8892, 14.8006, 55.8820 & 5 & 214 \\
glass & DW & 6.5920 & 0.0454, 0.0594, 0.2079, 3.5425, 9.8892, 14.8006, 55.8820 & 8 & 214 \\
ionosphere & DM & 0.9005 & 0.0389, 0.1133, 0.3176, 0.8849, 4.0262, 7.3003, 10.6989, 15.7211, 16.0338 & 10 & 351 \\
ionosphere & DW & 0.1133 & 0.0389, 0.3176, 0.8849, 0.9005, 4.0262, 7.3003, 10.6989, 15.7211, 16.0338 & 10 & 351 \\
iris & DM & 0.0482 & 0.1661, 1.5548, 1.5865, 3.5442, 4.0290, 5.6993, 7.0226, 17.7016, 63.2005 & 10 & 150 \\
iris & DW & 4.0290 & 0.0482, 0.1661, 1.5548, 1.5865, 3.5442, 5.6993, 7.0226, 17.7016, 63.2005 & 10 & 150 \\
letter & DM & 0.4779 & 0.3449, 0.4137 & 3 & 20000 \\
letter & DW & 0.2140 & 0.0360, 0.3449, 0.4137, 0.4779 & 5 & 20000 \\
magic04 & DM & 0.6279 & 0.2552, 0.3128 & 3 & 19020 \\
magic04 & DW & 0.3128 & 0.0611, 0.2552, 0.6279 & 4 & 19020 \\
optdigits & DM & 0.6042 & 0.2983 & 2 & 5620 \\
optdigits & DW & 0.6042 & 0.1262, 0.1584, 0.2983 & 4 & 5620 \\
page\_blocks & DM & 0.0551 & 0.0363, 0.0585 & 3 & 5473 \\
page\_blocks & DW & 0.0240 & 0.0064, 0.0076, 0.0363, 0.0551, 0.0585 & 6 & 5473 \\
pendigits & DM & 0.5211 & 0.2325, 0.4928 & 3 & 10992 \\
pendigits & DW & 0.4928 & 0.1355, 0.1541, 0.1793, 0.2325, 0.5211 & 6 & 10992 \\
phoneme & DM & 0.4225 &  & 1 & 5404 \\
phoneme & DW & 0.0842 & 0.0431, 0.0449, 0.1244, 0.1984, 0.4225 & 6 & 5404 \\
\end{tabular}

\end{table}

\clearpage
\section{Discussion}

\subsection{Summary of Findings}
We introduce a Dirichlet-weighted randomization framework for ensemble learning, with a focus on improving Random Forests.
The computational overhead of sampling Dirichlet weights is minimal.
By tuning the concentration parameter of the Dirichlet distribution, we control the dispersion of weights across trees, thereby diversifying bootstrap samples and reducing inter-tree correlation.
We also derive a soft threshold for the concentration parameter, approximately inversely proportional to the size of the training data, beyond which larger $\alpha$ values have little effect on the bootstrapped samples and the resulting learners.

Aggregated predictions were comparable over a reasonably wide range of $\alpha$ values (on a logarithmic scale) for most benchmark datasets.
Dirichlet-weighted (DW) Random Forests often surpassed strong baselines, whereas Dirichlet-Multinomial Bagging (DM) Random Forests was broadly competitive but not consistently superior.
These trends were observed across all three performance metrics.


\subsection{Future Work}
An interesting direction for future work is to investigate the conditions under which Dirichlet-weighted bagging is most beneficial---for instance, whether the gains depend on properties such as the nonlinearity of the predictor-response relationship or the degree of clustering in the training data. A deeper understanding of how randomized weighting influences predictive performance could enable the adaptive use of different randomization schemes.

\bibliography{refs}

\clearpage
\appendix
\thispagestyle{empty}

\onecolumn
\section{Missing Proofs}
\label{app:proofs}

\subsection{Proof of Lemma 2.}
$$
\mathbb{E}\|p-u\|_2^2=\sum_{i=1}^n \mathbb{E}\left(p_i-\frac{1}{n}\right)^2=\sum_{i=1}^n\left(\operatorname{Var}\left(p_i\right)+\left(\mathbb{E} p_i-\frac{1}{n}\right)^2\right) .
$$

For the symmetric Dirichlet, $\mathbb{E} p_i=\frac{\alpha}{n \alpha}=\frac{1}{n}$, so
$$
\mathbb{E}\|p-u\|_2^2=\sum_{i=1}^n \operatorname{Var}\left(p_i\right)=n \operatorname{Var}\left(p_1\right) .
$$
Under $\operatorname{Dir}\left(\alpha 1_n\right)$, each coordinate $p_i$ marginally has a Beta distribution:
$$
p_1 \sim \operatorname{Beta}(\alpha,(n-1) \alpha) .
$$
For $X \sim \operatorname{Beta}(a, b), \operatorname{Var}(X)=\frac{a b}{(a+b)^2(a+b+1)}$. Plugging $a=\alpha, b=(n-1) \alpha$,
$$
\operatorname{Var}\left(p_1\right)=\frac{\alpha(n-1) \alpha}{(n \alpha)^2(n \alpha+1)}=\frac{n-1}{n^2(n \alpha+1)}
$$

Therefore,
$$
\mathrm{E}\|p-u\|_2^2=n \cdot \frac{n-1}{n^2(n \alpha+1)}=\frac{n-1}{n(n \alpha+1)}
$$

\subsection{Derivation of $p(\alpha)$}\label{app:p-alpha-derivation}
Recall $p(\alpha)=\E[U/n]$, where $U$ is the number of distinct training points appearing at least once in a size-$n$ bag drawn via the Dirichlet–multinomial mechanism:
$\boldsymbol{p}\sim\Dir(\alpha\mathbf 1)$ and $(i_1,\ldots,i_n)\sim\mathrm{Multinomial}(n,\boldsymbol{p})$.

Fix an index $i\in\{1,\ldots,n\}$ and write $W:=p_i$.
By the marginal property of the Dirichlet,
\[
W \sim \mathrm{Beta}\!\big(\alpha,(n-1)\alpha\big).
\]

Conditioning on $W$, each of the $n$ draws picks index $i$ with probability $W$ independently.
Hence the probability that $i$ is excluded from the bag is
\[
\Pr(\text{$i$ excluded}\mid W)=(1-W)^n.
\]

Let $I_i=\mathbf{1}\{\text{$i$ appears at least once}\}$.
Then
\[
\E[I_i]
= \E\big[\Pr(I_i=1\mid W)\big]
= \E\big[1-(1-W)^n\big]
= 1-\E[(1-W)^n].
\]
By symmetry over $i$,
\[
p(\alpha)=\E\!\left[\frac{U}{n}\right]
= \frac{1}{n}\sum_{i=1}^n \E[I_i]
= \E[I_1]
= 1-\E[(1-W)^n].
\]

With $W\sim\mathrm{Beta}(\alpha,(n-1)\alpha)$ having density
\( f(w)=\dfrac{w^{\alpha-1}(1-w)^{(n-1)\alpha-1}}{B(\alpha,(n-1)\alpha)}\),
\[
\E[(1-W)^n]
= \int_0^1 (1-w)^n f(w)\,dw
= \frac{1}{B(\alpha,(n-1)\alpha)}\int_0^1 w^{\alpha-1}(1-w)^{(n-1)\alpha+n-1}\,dw
= \frac{B\!\big(\alpha,(n-1)\alpha+n\big)}{B\!\big(\alpha,(n-1)\alpha\big)},
\]
using the Beta function identity \(B(a,b)=\int_0^1 t^{a-1}(1-t)^{b-1}\,dt\).

\[
\;p(\alpha)\;=\;\E\!\left[\frac{U}{n}\right]
\;=\;1-\frac{B\!\big(\alpha,(n-1)\alpha+n\big)}{B\!\big(\alpha,(n-1)\alpha\big)}\;
\]

\textbf{Sanity checks.}
As $\alpha\to\infty$, $W\to 1/n$ so $p(\alpha)\to 1-(1-1/n)^n \approx 1-e^{-1}\;( \approx 0.632)$ (Efron bootstrap).
For $\alpha=1$, $p(\alpha)=\dfrac{n}{2n-1}\to \tfrac{1}{2}$ as $n\to\infty$.

\subsection{Note: BB/WLB equivalence}\label{app:bbwlb}
\subsubsection{Gamma-to-Dirichlet normalization}
Let $G_i \stackrel{ind}{\sim}\mathrm{Gamma}(\alpha,1)$ for $i=1,\dots,n$, and set
$T=\sum_{i=1}^n G_i$ and $W_i = G_i/T$. Then $\mathbf W=(W_1,\dots,W_n)\sim\Dir(\alpha\mathbf 1)$.
In particular, for $\alpha=1$, if $\tilde w_i\stackrel{iid}{\sim}\mathrm{Exp}(1)$ and
$w_i=\tilde w_i/\sum_j \tilde w_j$, then $\mathbf w\sim\Dir(1,\ldots,1)$.

\textit{Sketch.}
The joint density factorizes into a product of independent Gamma terms.
A standard change of variables from $(G_1,\dots,G_n)$ to $(T,W_1,\dots,W_{n-1})$
(with $W_n=1-\sum_{i<n}W_i$) yields the Dirichlet density on $\mathbf W$ and an
independent Gamma on $T$. The $\alpha=1$ case follows since $\mathrm{Exp}(1)=\mathrm{Gamma}(1,1)$.

\subsubsection{BB/WLB as $\alpha=1$ Dirichlet weighting}
For learners that accept per-observation weights, the Bayesian bootstrap (BB) and the
Weighted Likelihood Bootstrap (WLB) yield identical sampling distributions for the weight vector,
namely $\Dir(1,\ldots,1)$. Therefore, both are exactly our Dirichlet--weighted scheme at $\alpha=1$.

\subsection{Random Search for $\alpha$}\label{app:alpha_search}

Recall that $\textbf{w} \sim \operatorname{Dir}(\alpha \mathbf{1}_n)$, each coordinate has
$$\operatorname{Var}(w_i)
= \frac{n-1}{n^2\,(n\alpha+1)}
\propto \frac{1}{1+n\alpha}.$$
Thus, the dispersion scales approximately as $1/(1+n\alpha)$. To take steps that are 'uniform in effect', we work on the transformed axis $z := \log(1+n\alpha).$

Let the bounds $\alpha_{\min}$ and $\alpha_{\max}$ be given and define
\[
z_{\min}=\log\!\bigl(1+n\alpha_{\min}\bigr),
\qquad
z_{\max}=\log\!\bigl(1+n\alpha_{\max}\bigr).
\]
Sample $U \sim \operatorname{Unif}(0,1)$ and draw uniformly in $z$:
$$z \;=\; z_{\min} + U\,(z_{\max}-z_{\min})$$
The mapping to $\alpha$ gives the following result.
\begin{align*}
\alpha= \frac{\bigl(1+n\alpha_{\min}\bigr)\left(\dfrac{1+n\alpha_{\max}}{1+n\alpha_{\min}}\right)^{U} - 1}{n}
\end{align*}

\subsection{A Tighter Bound Analysis for $\alpha_{RF}$}
\subsubsection{High-probability bound in $\ell_\infty$ sense} Let $\mathbf w\sim\Dir(\alpha\,\mathbf 1_n)$ with $\alpha>0$, and fix $\varepsilon>0$, $\delta\in(0,1)$.
If
\[
\alpha \;\ge\; \frac{\,n-1\,}{\,n\,\varepsilon^{2}\,\delta\,}\;-\;\frac{1}{n},
\]
then $\mathbf w$ is $\epsilon$ -uniform weights in the $\ell_\infty$ sense with probability at least $1-\delta$.

\textit{Proof.}
By Chebyshev, for each $i$,
$\Pr\!\big(|w_i-\tfrac1n|>\varepsilon\big)\le \Var(w_i)/\varepsilon^2$.
For $\Dir(\alpha\,\mathbf 1_n)$, $\Var(w_i)=\dfrac{n-1}{n^2(n\alpha+1)}$.
A union bound to $i=1,\dots,n$ yields
\[
\Pr\!\Big(\max_i|w_i-\tfrac1n|>\varepsilon\Big)
\ \le\ \frac{n\,\Var(w_i)}{\varepsilon^2}
\ =\ \frac{n-1}{n(n\alpha+1)\,\varepsilon^2}\,.
\]
Requiring the right-hand side to be $\le \delta$ and solving for $\alpha$ establishes the claim.

\subsubsection{Expectation bound in $\ell_\infty$} Let $\mathbf w\sim\Dir(\alpha\,\mathbf 1_n)$ and $\mathbf u=\tfrac1n\mathbf 1_n$. Then
\[
\E\!\;\|\mathbf w - \mathbf u\|_\infty\;
\;\le\; 2\,\sqrt{\frac{n-1}{\,n\,(n\alpha+1)\,}}\,.
\]
In particular, we obtain $\E\|\mathbf w - \mathbf u\|_\infty \le \varepsilon$ provided that
\begin{equation}
\label{exption-bound}
\alpha \;\ge\; \frac{\,4(n-1)\,}{\,n^2\,\varepsilon^2\,}\;-\;\frac{1}{n}.
\end{equation}

\textit{Proof.}
Suppose $X=\|\mathbf w - \mathbf u\|_\infty$, we have
\[
\E[X]=\int_0^\infty \Pr(X>t)\,dt
\;\le\; \int_0^\infty \min\!\Big\{1,\ \frac{A}{t^2}\Big\}\,dt
\]
\(
\text{with } A=\frac{n-1}{n(n\alpha+1)}.
\)
Splitting at $t_0=\sqrt{A}$ gives $\int_0^{t_0}1\,dt+\int_{t_0}^\infty A t^{-2}\,dt
= \sqrt{A}+A/\sqrt{A}=2\sqrt{A}$, proving the bound. The condition \eqref{exption-bound} follows by requiring $2\sqrt{A}\le\varepsilon$.

\section{Additional Experimental Results}
\label{app:additional-results}

\subsection{Dataset Information}\label{app:dataset-info}

\begin{table}[h]
\caption{Dataset Statistics} \label{tab:datasets}
\begin{center}
\begin{tabular}{lcccc}
\textbf{Dataset} & \boldmath$n$ & \boldmath$C$ & \textbf{Imbalance} & \boldmath$\alpha_{\mathrm{RF}}(n,\varepsilon{=}0.01)$ \\
\hline \\
Adult & 48{,}842 & 2 & high & 0.205 \\
Banknote & 1{,}372 & 2 & low  & 7.283 \\
Breast Cancer (WDBC) & 569 & 2 & mid & 17.542 \\
Credit-G (German) & 1{,}000 & 2 & high & 9.989 \\
Digits (sklearn) & 1{,}797 & 10 & low & 5.561 \\
Ecoli & 336 & 8 & high & 29.670 \\
Glass & 214 & 6 & high & 46.506 \\
Ionosphere & 351 & 2 & mid & 28.406 \\
Iris & 150 & 3 & low & 66.216 \\
Letter & 20{,}000 & 26 & low & 0.500 \\
MAGIC Gamma (magic04) & 19{,}020 & 2 & high & 0.526 \\
Optdigits & 5{,}620 & 10 & low & 1.779 \\
Page Blocks & 5{,}473 & 5 & high & 1.827 \\
Pendigits & 10{,}992 & 10 & low & 0.910 \\
Phoneme & 5{,}404 & 2 & mid & 1.850 \\
\end{tabular}
\end{center}
\end{table}

\subsection{Experiment Metrics on Classifications over 10 Seeds}

\begin{table}[h]
\caption{Results for Adult} \label{tab:adult}
\begin{center}
\begin{tabular}{llllllll}
\textbf{method} &\textbf{alpha} &\textbf{accuracy} &\textbf{logloss} &\textbf{auroc} &\textbf{n} &\textbf{d} &\textbf{C} \\
\hline \\
ET &  & 0.8293±0.0009 & 0.5109±0.0078 & 0.8779±0.0010 & 45222 & 105 & 2 \\
RF-bootstrap &  & 0.8491±0.0011 & 0.3573±0.0034 & 0.9018±0.0010 & 45222 & 105 & 2 \\
RF-no-bootstrap &  & 0.8431±0.0010 & 0.4893±0.0077 & 0.8870±0.0010 & 45222 & 105 & 2 \\
Subsample &  & 0.8499±0.0011 & 0.3616±0.0047 & 0.9015±0.0011 & 45222 & 105 & 2 \\
Bayesian RF & 1.0 & 0.8436±0.0009 & 0.4576±0.0093 & 0.8894±0.0010 & 45222 & 105 & 2 \\
DM & 0 & 0.7808±0.0032 & 0.4286±0.0020 & 0.8708±0.0026 & 45222 & 105 & 2 \\
DM & 0.0001 & 0.8152±0.0009 & 0.3969±0.0011 & 0.8806±0.0011 & 45222 & 105 & 2 \\
DM & 0.0003 & 0.8283±0.0010 & 0.3732±0.0011 & 0.8923±0.0014 & 45222 & 105 & 2 \\
DM & 0.0004 & 0.8307±0.0014 & 0.3706±0.0014 & 0.8924±0.0014 & 45222 & 105 & 2 \\
DM & 0.0009 & 0.8378±0.0014 & 0.3559±0.0013 & 0.8982±0.0013 & 45222 & 105 & 2 \\
DM & 0.0066 & 0.8523±0.0011 & 0.3321±0.0018 & 0.9080±0.0010 & 45222 & 105 & 2 \\
DM & 0.0090 & 0.8537±0.0009 & 0.3287±0.0012 & 0.9087±0.0010 & 45222 & 105 & 2 \\
DM & 0.0133 & 0.8556±0.0009 & 0.3258±0.0017 & 0.9097±0.0009 & 45222 & 105 & 2 \\
DM & 0.1573 & 0.8561±0.0011 & 0.3269±0.0026 & 0.9100±0.0010 & 45222 & 105 & 2 \\
DM & 0.1744 & 0.8559±0.0013 & 0.3271±0.0037 & 0.9100±0.0011 & 45222 & 105 & 2 \\
DW & 0 & 0.8187±0.0011 & 0.3926±0.0009 & 0.8835±0.0017 & 45222 & 105 & 2 \\
DW & 0.0001 & 0.8301±0.0018 & 0.3711±0.0013 & 0.8917±0.0016 & 45222 & 105 & 2 \\
DW & 0.0003 & 0.8394±0.0013 & 0.3539±0.0012 & 0.8987±0.0013 & 45222 & 105 & 2 \\
DW & 0.0004 & 0.8403±0.0013 & 0.3517±0.0009 & 0.8998±0.0011 & 45222 & 105 & 2 \\
DW & 0.0009 & 0.8463±0.0011 & 0.3413±0.0012 & 0.9037±0.0011 & 45222 & 105 & 2 \\
DW & 0.0066 & 0.8561±0.0009 & 0.3216±0.0015 & 0.9102±0.0010 & 45222 & 105 & 2 \\
DW & 0.0090 & 0.8563±0.0013 & 0.3198±0.0015 & 0.9106±0.0010 & 45222 & 105 & 2 \\
DW & 0.0133 & 0.8567±0.0011 & 0.3178±0.0015 & 0.9110±0.0010 & 45222 & 105 & 2 \\
DW & 0.1573 & 0.8432±0.0010 & 0.3988±0.0039 & 0.8908±0.0010 & 45222 & 105 & 2 \\
DW & 0.1744 & 0.8436±0.0010 & 0.4056±0.0036 & 0.8906±0.0010 & 45222 & 105 & 2 \\

\end{tabular}
\end{center}
\end{table}

\begin{table}[h]
\caption{Results for Banknote} \label{tab:banknote}
\begin{center}
\begin{tabular}{llllllll}
\textbf{method} &\textbf{alpha} &\textbf{accuracy} &\textbf{logloss} &\textbf{auroc} &\textbf{n} &\textbf{d} &\textbf{C} \\
\hline \\
ET &  & 0.9993±0.0005 & 0.0158±0.0011 & 1.0000±0.0000 & 1372 & 4 & 2 \\
RF-bootstrap &  & 0.9938±0.0012 & 0.0309±0.0022 & 0.9998±0.0001 & 1372 & 4 & 2 \\
RF-no-bootstrap &  & 0.9945±0.0011 & 0.0240±0.0020 & 0.9998±0.0001 & 1372 & 4 & 2 \\
Subsample &  & 0.9942±0.0006 & 0.0304±0.0022 & 0.9997±0.0001 & 1372 & 4 & 2 \\
Bayesian RF & 1.0 & 0.9945±0.0011 & 0.0229±0.0019 & 0.9999±0.0001 & 1372 & 4 & 2 \\
DM & 0.0007 & 0.8898±0.0069 & 0.4070±0.0065 & 0.9691±0.0027 & 1372 & 4 & 2 \\
DM & 0.0023 & 0.9280±0.0060 & 0.2803±0.0054 & 0.9815±0.0029 & 1372 & 4 & 2 \\
DM & 0.0047 & 0.9484±0.0059 & 0.2211±0.0056 & 0.9892±0.0017 & 1372 & 4 & 2 \\
DM & 0.0495 & 0.9836±0.0026 & 0.0868±0.0038 & 0.9990±0.0003 & 1372 & 4 & 2 \\
DM & 0.2862 & 0.9916±0.0011 & 0.0470±0.0032 & 0.9996±0.0001 & 1372 & 4 & 2 \\
DM & 0.8498 & 0.9931±0.0008 & 0.0370±0.0025 & 0.9997±0.0001 & 1372 & 4 & 2 \\
DM & 1.1743 & 0.9935±0.0011 & 0.0349±0.0023 & 0.9998±0.0001 & 1372 & 4 & 2 \\
DM & 1.5397 & 0.9942±0.0008 & 0.0346±0.0024 & 0.9997±0.0001 & 1372 & 4 & 2 \\
DM & 2.3156 & 0.9935±0.0009 & 0.0332±0.0022 & 0.9997±0.0001 & 1372 & 4 & 2 \\
DM & 4.8365 & 0.9945±0.0006 & 0.0315±0.0022 & 0.9998±0.0001 & 1372 & 4 & 2 \\
DW & 0.0007 & 0.9295±0.0067 & 0.2214±0.0053 & 0.9889±0.0015 & 1372 & 4 & 2 \\
DW & 0.0023 & 0.9658±0.0044 & 0.1342±0.0050 & 0.9971±0.0006 & 1372 & 4 & 2 \\
DW & 0.0047 & 0.9858±0.0027 & 0.0932±0.0037 & 0.9994±0.0002 & 1372 & 4 & 2 \\
DW & 0.0495 & 0.9971±0.0009 & 0.0280±0.0017 & 1.0000±0.0000 & 1372 & 4 & 2 \\
DW & 0.2862 & 0.9956±0.0007 & 0.0232±0.0016 & 1.0000±0.0000 & 1372 & 4 & 2 \\
DW & 0.8498 & 0.9953±0.0009 & 0.0222±0.0017 & 0.9999±0.0000 & 1372 & 4 & 2 \\
DW & 1.1743 & 0.9945±0.0010 & 0.0226±0.0018 & 0.9999±0.0000 & 1372 & 4 & 2 \\
DW & 1.5397 & 0.9949±0.0008 & 0.0230±0.0020 & 0.9998±0.0001 & 1372 & 4 & 2 \\
DW & 2.3156 & 0.9938±0.0009 & 0.0234±0.0018 & 0.9998±0.0000 & 1372 & 4 & 2 \\
DW & 4.8365 & 0.9949±0.0011 & 0.0229±0.0019 & 0.9998±0.0001 & 1372 & 4 & 2 \\

\end{tabular}
\end{center}
\end{table}

\begin{table}[h]
\caption{Results for Breast Cancer} \label{tab:breast-cancer}
\begin{center}
\begin{tabular}{llllllll}
\textbf{method} &\textbf{alpha} &\textbf{accuracy} &\textbf{logloss} &\textbf{auroc} &\textbf{n} &\textbf{d} &\textbf{C} \\
\hline \\
ET &  & 0.9649±0.0035 & 0.0982±0.0061 & 0.9941±0.0016 & 569 & 30 & 2 \\
RF-bootstrap &  & 0.9649±0.0039 & 0.2130±0.0507 & 0.9893±0.0033 & 569 & 30 & 2 \\
RF-no-bootstrap &  & 0.9667±0.0025 & 0.1541±0.0431 & 0.9913±0.0027 & 569 & 30 & 2 \\
Subsample &  & 0.9684±0.0035 & 0.1304±0.0305 & 0.9922±0.0024 & 569 & 30 & 2 \\
Bayesian RF & 1.0 & 0.9719±0.0037 & 0.1278±0.0345 & 0.9921±0.0028 & 569 & 30 & 2 \\
DM & 0.0012 & 0.9298±0.0070 & 0.3449±0.0057 & 0.9892±0.0024 & 569 & 30 & 2 \\
DM & 0.0014 & 0.9439±0.0069 & 0.3290±0.0066 & 0.9878±0.0025 & 569 & 30 & 2 \\
DM & 0.0255 & 0.9570±0.0048 & 0.1410±0.0069 & 0.9895±0.0029 & 569 & 30 & 2 \\
DM & 0.0570 & 0.9632±0.0043 & 0.1281±0.0078 & 0.9900±0.0028 & 569 & 30 & 2 \\
DM & 0.1427 & 0.9623±0.0035 & 0.1173±0.0084 & 0.9909±0.0026 & 569 & 30 & 2 \\
DM & 0.2972 & 0.9640±0.0036 & 0.1132±0.0086 & 0.9915±0.0025 & 569 & 30 & 2 \\
DM & 2.2321 & 0.9658±0.0042 & 0.1597±0.0429 & 0.9909±0.0030 & 569 & 30 & 2 \\
DM & 5.7722 & 0.9684±0.0035 & 0.1597±0.0433 & 0.9908±0.0030 & 569 & 30 & 2 \\
DM & 6.4755 & 0.9667±0.0025 & 0.1604±0.0401 & 0.9908±0.0029 & 569 & 30 & 2 \\
DM & 10.1280 & 0.9640±0.0044 & 0.1582±0.0436 & 0.9914±0.0030 & 569 & 30 & 2 \\
DW & 0.0012 & 0.9404±0.0060 & 0.1752±0.0061 & 0.9897±0.0022 & 569 & 30 & 2 \\
DW & 0.0014 & 0.9456±0.0047 & 0.1667±0.0069 & 0.9900±0.0023 & 569 & 30 & 2 \\
DW & 0.0255 & 0.9649±0.0045 & 0.1258±0.0254 & 0.9918±0.0026 & 569 & 30 & 2 \\
DW & 0.0570 & 0.9693±0.0042 & 0.1031±0.0078 & 0.9930±0.0020 & 569 & 30 & 2 \\
DW & 0.1427 & 0.9684±0.0048 & 0.1637±0.0362 & 0.9904±0.0031 & 569 & 30 & 2 \\
DW & 0.2972 & 0.9649±0.0039 & 0.1559±0.0433 & 0.9914±0.0031 & 569 & 30 & 2 \\
DW & 2.2321 & 0.9675±0.0045 & 0.1265±0.0341 & 0.9919±0.0026 & 569 & 30 & 2 \\
DW & 5.7722 & 0.9702±0.0033 & 0.2083±0.0513 & 0.9900±0.0033 & 569 & 30 & 2 \\
DW & 6.4755 & 0.9675±0.0041 & 0.1806±0.0493 & 0.9907±0.0031 & 569 & 30 & 2 \\
DW & 10.1280 & 0.9684±0.0037 & 0.2085±0.0506 & 0.9900±0.0031 & 569 & 30 & 2 \\

\end{tabular}
\end{center}
\end{table}

\begin{table}[h]
\caption{Results for Credit German} \label{tab:credit-g}
\begin{center}
\begin{tabular}{llllllll}
\textbf{method} &\textbf{alpha} &\textbf{accuracy} &\textbf{logloss} &\textbf{auroc} &\textbf{n} &\textbf{d} &\textbf{C} \\
\hline \\
ET &  & 0.7510±0.0084 & 0.5024±0.0106 & 0.7741±0.0128 & 1000 & 61 & 2 \\
RF-bootstrap &  & 0.7650±0.0098 & 0.4941±0.0099 & 0.7868±0.0147 & 1000 & 61 & 2 \\
RF-no-bootstrap &  & 0.7540±0.0078 & 0.4953±0.0104 & 0.7820±0.0136 & 1000 & 61 & 2 \\
Subsample &  & 0.7615±0.0080 & 0.4908±0.0092 & 0.7910±0.0137 & 1000 & 61 & 2 \\
Bayesian RF & 1.0 & 0.7665±0.0080 & 0.4921±0.0093 & 0.7902±0.0141 & 1000 & 61 & 2 \\
DM & 0.0005 & 0.7000±0.0000 & 0.5931±0.0031 & 0.6860±0.0206 & 1000 & 61 & 2 \\
DM & 0.0009 & 0.7000±0.0000 & 0.5819±0.0030 & 0.7126±0.0183 & 1000 & 61 & 2 \\
DM & 0.0022 & 0.6995±0.0012 & 0.5662±0.0028 & 0.7454±0.0129 & 1000 & 61 & 2 \\
DM & 0.0072 & 0.7035±0.0021 & 0.5478±0.0039 & 0.7623±0.0144 & 1000 & 61 & 2 \\
DM & 0.2543 & 0.7620±0.0053 & 0.5024±0.0074 & 0.7904±0.0140 & 1000 & 61 & 2 \\
DM & 0.4833 & 0.7615±0.0080 & 0.4986±0.0079 & 0.7904±0.0141 & 1000 & 61 & 2 \\
DM & 2.4057 & 0.7680±0.0070 & 0.4927±0.0097 & 0.7904±0.0153 & 1000 & 61 & 2 \\
DM & 2.4879 & 0.7710±0.0077 & 0.4923±0.0090 & 0.7944±0.0139 & 1000 & 61 & 2 \\
DM & 3.2354 & 0.7585±0.0087 & 0.4931±0.0095 & 0.7917±0.0145 & 1000 & 61 & 2 \\
DM & 3.5595 & 0.7660±0.0089 & 0.4926±0.0087 & 0.7931±0.0135 & 1000 & 61 & 2 \\
DW & 0.0005 & 0.7005±0.0005 & 0.5598±0.0035 & 0.7510±0.0130 & 1000 & 61 & 2 \\
DW & 0.0009 & 0.7010±0.0010 & 0.5460±0.0044 & 0.7664±0.0160 & 1000 & 61 & 2 \\
DW & 0.0022 & 0.7050±0.0025 & 0.5348±0.0054 & 0.7734±0.0153 & 1000 & 61 & 2 \\
DW & 0.0072 & 0.7350±0.0050 & 0.5146±0.0050 & 0.7922±0.0129 & 1000 & 61 & 2 \\
DW & 0.2543 & 0.7600±0.0056 & 0.4944±0.0086 & 0.7897±0.0142 & 1000 & 61 & 2 \\
DW & 0.4833 & 0.7610±0.0090 & 0.4917±0.0093 & 0.7932±0.0139 & 1000 & 61 & 2 \\
DW & 2.4057 & 0.7570±0.0094 & 0.4919±0.0105 & 0.7881±0.0143 & 1000 & 61 & 2 \\
DW & 2.4879 & 0.7595±0.0117 & 0.4928±0.0106 & 0.7853±0.0143 & 1000 & 61 & 2 \\
DW & 3.2354 & 0.7615±0.0093 & 0.4918±0.0110 & 0.7881±0.0145 & 1000 & 61 & 2 \\
DW & 3.5595 & 0.7590±0.0103 & 0.4928±0.0099 & 0.7856±0.0136 & 1000 & 61 & 2 \\

\end{tabular}
\end{center}
\end{table}

\begin{table}[h]
\caption{Results for Digits} \label{tab:digits}
\begin{center}
\begin{tabular}{llllllll}
\textbf{method} &\textbf{alpha} &\textbf{accuracy} &\textbf{logloss} &\textbf{auroc} &\textbf{n} &\textbf{d} &\textbf{C} \\
\hline \\
ET &  & 0.9839±0.0015 & 0.2840±0.0025 & 0.9996±0.0001 & 1797 & 64 & 10 \\
RF-bootstrap &  & 0.9769±0.0013 & 0.3098±0.0031 & 0.9994±0.0001 & 1797 & 64 & 10 \\
RF-no-bootstrap &  & 0.9781±0.0017 & 0.2588±0.0033 & 0.9995±0.0001 & 1797 & 64 & 10 \\
Subsample &  & 0.9767±0.0022 & 0.3039±0.0025 & 0.9995±0.0001 & 1797 & 64 & 10 \\
Bayesian RF & 1.0 & 0.9797±0.0021 & 0.2699±0.0029 & 0.9996±0.0001 & 1797 & 64 & 10 \\
DM & 0 & 0.1692±0.0183 & 2.2434±0.0044 & 0.7839±0.0081 & 1797 & 64 & 10 \\
DM & 0.0001 & 0.5397±0.0182 & 2.0021±0.0075 & 0.9198±0.0051 & 1797 & 64 & 10 \\
DM & 0.0005 & 0.7697±0.0077 & 1.7318±0.0054 & 0.9616±0.0021 & 1797 & 64 & 10 \\
DM & 0.0062 & 0.9306±0.0026 & 0.9965±0.0043 & 0.9940±0.0004 & 1797 & 64 & 10 \\
DM & 0.0086 & 0.9403±0.0022 & 0.9087±0.0054 & 0.9951±0.0005 & 1797 & 64 & 10 \\
DM & 0.0098 & 0.9397±0.0030 & 0.8793±0.0045 & 0.9949±0.0004 & 1797 & 64 & 10 \\
DM & 0.5192 & 0.9742±0.0017 & 0.3598±0.0034 & 0.9992±0.0001 & 1797 & 64 & 10 \\
DM & 1.0745 & 0.9744±0.0021 & 0.3371±0.0037 & 0.9994±0.0001 & 1797 & 64 & 10 \\
DM & 1.1195 & 0.9781±0.0017 & 0.3349±0.0032 & 0.9995±0.0001 & 1797 & 64 & 10 \\
DM & 4.8044 & 0.9756±0.0017 & 0.3154±0.0032 & 0.9995±0.0001 & 1797 & 64 & 10 \\
DW & 0 & 0.5442±0.0223 & 1.9910±0.0066 & 0.9259±0.0050 & 1797 & 64 & 10 \\
DW & 0.0001 & 0.8633±0.0075 & 1.4480±0.0063 & 0.9826±0.0009 & 1797 & 64 & 10 \\
DW & 0.0005 & 0.9153±0.0033 & 1.1264±0.0063 & 0.9916±0.0007 & 1797 & 64 & 10 \\
DW & 0.0062 & 0.9689±0.0017 & 0.5460±0.0042 & 0.9987±0.0002 & 1797 & 64 & 10 \\
DW & 0.0086 & 0.9700±0.0018 & 0.4973±0.0041 & 0.9987±0.0002 & 1797 & 64 & 10 \\
DW & 0.0098 & 0.9725±0.0013 & 0.4801±0.0043 & 0.9990±0.0002 & 1797 & 64 & 10 \\
DW & 0.5192 & 0.9817±0.0012 & 0.2764±0.0030 & 0.9996±0.0001 & 1797 & 64 & 10 \\
DW & 1.0745 & 0.9814±0.0019 & 0.2680±0.0032 & 0.9996±0.0001 & 1797 & 64 & 10 \\
DW & 1.1195 & 0.9789±0.0020 & 0.2679±0.0027 & 0.9995±0.0001 & 1797 & 64 & 10 \\
DW & 4.8044 & 0.9789±0.0014 & 0.2611±0.0035 & 0.9996±0.0001 & 1797 & 64 & 10 \\

\end{tabular}
\end{center}
\end{table}

\begin{table}[h]
\caption{Results for Ecoli} \label{tab:ecoli}
\begin{center}
\begin{tabular}{llllllll}
\textbf{method} &\textbf{alpha} &\textbf{accuracy} &\textbf{logloss} &\textbf{auroc} &\textbf{n} &\textbf{d} &\textbf{C} \\
\hline \\
ET &  & 0.8809±0.0071 & 0.3900±0.0168 &  & 336 & 7 & 8 \\
RF-bootstrap &  & 0.8941±0.0098 & 0.4288±0.0550 &  & 336 & 7 & 8 \\
RF-no-bootstrap &  & 0.8824±0.0076 & 0.5224±0.0866 &  & 336 & 7 & 8 \\
Bayesian RF & 1.0 & 0.8897±0.0077 & 0.4656±0.0646 &  & 336 & 7 & 8 \\
Subsample &  & 0.8985±0.0080 & 0.5175±0.0839 &  & 336 & 7 & 8 \\
DM & 0.0236 & 0.8544±0.0083 & 0.5448±0.0083 &  & 336 & 7 & 8 \\
DM & 0.1338 & 0.8956±0.0077 & 0.4217±0.0133 &  & 336 & 7 & 8 \\
DM & 0.1462 & 0.8971±0.0073 & 0.4193±0.0128 &  & 336 & 7 & 8 \\
DM & 0.2199 & 0.9000±0.0092 & 0.4066±0.0149 &  & 336 & 7 & 8 \\
DM & 2.1998 & 0.8985±0.0077 & 0.4264±0.0554 &  & 336 & 7 & 8 \\
DM & 2.7826 & 0.8956±0.0083 & 0.3794±0.0162 &  & 336 & 7 & 8 \\
DM & 4.2772 & 0.9000±0.0072 & 0.4228±0.0500 &  & 336 & 7 & 8 \\
DM & 17.1033 & 0.8985±0.0083 & 0.4222±0.0497 &  & 336 & 7 & 8 \\
DM & 18.8559 & 0.8956±0.0083 & 0.5200±0.0829 &  & 336 & 7 & 8 \\
DM & 25.8465 & 0.8971±0.0082 & 0.3730±0.0154 &  & 336 & 7 & 8 \\
DW & 0.0236 & 0.8971±0.0076 & 0.3961±0.0135 &  & 336 & 7 & 8 \\
DW & 0.1338 & 0.8912±0.0080 & 0.3825±0.0215 &  & 336 & 7 & 8 \\
DW & 0.1462 & 0.8882±0.0080 & 0.4702±0.0698 &  & 336 & 7 & 8 \\
DW & 0.2199 & 0.8897±0.0083 & 0.4325±0.0499 &  & 336 & 7 & 8 \\
DW & 2.1998 & 0.8824±0.0079 & 0.5171±0.0878 &  & 336 & 7 & 8 \\
DW & 2.7826 & 0.8897±0.0070 & 0.4730±0.0810 &  & 336 & 7 & 8 \\
DW & 4.2772 & 0.8809±0.0077 & 0.5155±0.0855 &  & 336 & 7 & 8 \\
DW & 17.1033 & 0.8735±0.0085 & 0.5219±0.0854 &  & 336 & 7 & 8 \\
DW & 18.8559 & 0.8779±0.0082 & 0.5240±0.0829 &  & 336 & 7 & 8 \\
DW & 25.8465 & 0.8809±0.0080 & 0.5179±0.0860 &  & 336 & 7 & 8 \\

\end{tabular}
\end{center}
\end{table}

\begin{table}[h]
\caption{Results for Glass} \label{tab:glass}
\begin{center}
\begin{tabular}{llllllll}
\textbf{method} &\textbf{alpha} &\textbf{accuracy} &\textbf{logloss} &\textbf{auroc} &\textbf{n} &\textbf{d} &\textbf{C} \\
\hline \\
ET &  & 0.7907±0.0125 & 0.6105±0.0215 & 0.9591±0.0044 & 214 & 9 & 6 \\
RF-bootstrap &  & 0.7767±0.0131 & 0.6565±0.0285 & 0.9492±0.0050 & 214 & 9 & 6 \\
RF-no-bootstrap &  & 0.7698±0.0203 & 0.7119±0.0993 & 0.9486±0.0057 & 214 & 9 & 6 \\
Subsample &  & 0.7837±0.0155 & 0.6526±0.0299 & 0.9475±0.0054 & 214 & 9 & 6 \\
Bayesian RF & 1.0 & 0.7907±0.0177 & 0.6154±0.0320 & 0.9525±0.0055 & 214 & 9 & 6 \\
DM & 0.0012 & 0.4186±0.0222 & 1.3903±0.0066 & 0.8357±0.0116 & 214 & 9 & 6 \\
DM & 0.0029 & 0.4558±0.0174 & 1.2868±0.0042 & 0.8509±0.0101 & 214 & 9 & 6 \\
DM & 0.0453 & 0.6628±0.0184 & 0.8780±0.0152 & 0.9294±0.0049 & 214 & 9 & 6 \\
DM & 0.0593 & 0.6512±0.0208 & 0.8510±0.0189 & 0.9293±0.0084 & 214 & 9 & 6 \\
DM & 0.2078 & 0.7326±0.0220 & 0.7336±0.0210 & 0.9424±0.0066 & 214 & 9 & 6 \\
DM & 3.5425 & 0.7767±0.0148 & 0.6552±0.0269 & 0.9531±0.0045 & 214 & 9 & 6 \\
DM & 6.5919 & 0.7791±0.0144 & 0.6569±0.0281 & 0.9483±0.0053 & 214 & 9 & 6 \\
DM & 9.8891 & 0.7767±0.0181 & 0.6514±0.0305 & 0.9507±0.0052 & 214 & 9 & 6 \\
DM & 14.8005 & 0.7767±0.0140 & 0.6506±0.0281 & 0.9540±0.0054 & 214 & 9 & 6 \\
DM & 55.8819 & 0.7744±0.0163 & 0.6450±0.0295 & 0.9525±0.0052 & 214 & 9 & 6 \\
DW & 0.0012 & 0.6023±0.0164 & 1.0855±0.0129 & 0.8931±0.0081 & 214 & 9 & 6 \\
DW & 0.0029 & 0.6395±0.0160 & 0.9449±0.0154 & 0.9132±0.0070 & 214 & 9 & 6 \\
DW & 0.0453 & 0.7767±0.0160 & 0.6330±0.0226 & 0.9523±0.0065 & 214 & 9 & 6 \\
DW & 0.0593 & 0.7605±0.0159 & 0.6195±0.0235 & 0.9571±0.0051 & 214 & 9 & 6 \\
DW & 0.2078 & 0.7814±0.0174 & 0.6129±0.0262 & 0.9572±0.0050 & 214 & 9 & 6 \\
DW & 3.5425 & 0.7837±0.0211 & 0.7710±0.1680 & 0.9506±0.0055 & 214 & 9 & 6 \\
DW & 6.5919 & 0.7837±0.0174 & 0.7698±0.1660 & 0.9502±0.0054 & 214 & 9 & 6 \\
DW & 9.8891 & 0.7628±0.0179 & 0.8601±0.1844 & 0.9502±0.0057 & 214 & 9 & 6 \\
DW & 14.8005 & 0.7814±0.0174 & 0.7099±0.1002 & 0.9505±0.0056 & 214 & 9 & 6 \\
DW & 55.8819 & 0.7581±0.0163 & 0.7045±0.0978 & 0.9503±0.0059 & 214 & 9 & 6 \\

\end{tabular}
\end{center}
\end{table}

\begin{table}[h]
\caption{Results for Ionosphere} \label{tab:ionosphere}
\begin{center}
\begin{tabular}{llllllll}
\textbf{method} &\textbf{alpha} &\textbf{accuracy} &\textbf{logloss} &\textbf{auroc} &\textbf{n} &\textbf{d} &\textbf{C} \\
\hline \\
ET &  & 0.9507±0.0057 & 0.1632±0.0072 & 0.9917±0.0019 & 351 & 34 & 2 \\
RF-bootstrap &  & 0.9423±0.0065 & 0.1830±0.0110 & 0.9828±0.0036 & 351 & 34 & 2 \\
RF-no-bootstrap &  & 0.9380±0.0079 & 0.1639±0.0125 & 0.9847±0.0034 & 351 & 34 & 2 \\
Subsample &  & 0.9451±0.0074 & 0.1773±0.0103 & 0.9840±0.0033 & 351 & 34 & 2 \\
Bayesian RF & 1.0 & 0.9451±0.0074 & 0.1723±0.0109 & 0.9849±0.0034 & 351 & 34 & 2 \\
DM & 0.0389 & 0.9380±0.0082 & 0.2895±0.0075 & 0.9735±0.0048 & 351 & 34 & 2 \\
DM & 0.1133 & 0.9380±0.0073 & 0.2385±0.0084 & 0.9769±0.0047 & 351 & 34 & 2 \\
DM & 0.3176 & 0.9465±0.0062 & 0.2043±0.0101 & 0.9804±0.0041 & 351 & 34 & 2 \\
DM & 0.8848 & 0.9437±0.0076 & 0.1919±0.0101 & 0.9824±0.0037 & 351 & 34 & 2 \\
DM & 0.9005 & 0.9465±0.0062 & 0.1908±0.0086 & 0.9820±0.0031 & 351 & 34 & 2 \\
DM & 4.0261 & 0.9451±0.0068 & 0.1828±0.0099 & 0.9828±0.0035 & 351 & 34 & 2 \\
DM & 7.3003 & 0.9437±0.0079 & 0.1798±0.0103 & 0.9831±0.0033 & 351 & 34 & 2 \\
DM & 10.6988 & 0.9465±0.0066 & 0.1820±0.0104 & 0.9827±0.0035 & 351 & 34 & 2 \\
DM & 15.7211 & 0.9437±0.0066 & 0.1805±0.0110 & 0.9841±0.0034 & 351 & 34 & 2 \\
DM & 16.0337 & 0.9423±0.0074 & 0.1817±0.0102 & 0.9833±0.0034 & 351 & 34 & 2 \\
DW & 0.0389 & 0.9451±0.0068 & 0.1926±0.0078 & 0.9870±0.0029 & 351 & 34 & 2 \\
DW & 0.1133 & 0.9507±0.0067 & 0.1821±0.0087 & 0.9870±0.0030 & 351 & 34 & 2 \\
DW & 0.3176 & 0.9408±0.0075 & 0.1770±0.0097 & 0.9858±0.0034 & 351 & 34 & 2 \\
DW & 0.8848 & 0.9437±0.0066 & 0.1717±0.0102 & 0.9852±0.0030 & 351 & 34 & 2 \\
DW & 0.9005 & 0.9394±0.0076 & 0.1704±0.0094 & 0.9852±0.0031 & 351 & 34 & 2 \\
DW & 4.0261 & 0.9408±0.0078 & 0.1683±0.0123 & 0.9840±0.0036 & 351 & 34 & 2 \\
DW & 7.3003 & 0.9423±0.0080 & 0.1640±0.0106 & 0.9850±0.0030 & 351 & 34 & 2 \\
DW & 10.6988 & 0.9479±0.0063 & 0.1619±0.0110 & 0.9860±0.0030 & 351 & 34 & 2 \\
DW & 15.7211 & 0.9366±0.0076 & 0.1628±0.0121 & 0.9850±0.0034 & 351 & 34 & 2 \\
DW & 16.0337 & 0.9380±0.0082 & 0.1671±0.0118 & 0.9837±0.0034 & 351 & 34 & 2 \\

\end{tabular}
\end{center}
\end{table}

\begin{table}[h]
\caption{Results for Iris} \label{tab:iris}
\begin{center}
\begin{tabular}{llllllll}
\textbf{method} &\textbf{alpha} &\textbf{accuracy} &\textbf{logloss} &\textbf{auroc} &\textbf{n} &\textbf{d} &\textbf{C} \\
\hline \\
ET &  & 0.9600±0.0097 & 0.0928±0.0123 & 0.9966±0.0017 & 150 & 4 & 3 \\
RF-bootstrap &  & 0.9500±0.0124 & 0.1160±0.0233 & 0.9950±0.0027 & 150 & 4 & 3 \\
RF-no-bootstrap &  & 0.9433±0.0071 & 0.2316±0.1249 & 0.9924±0.0038 & 150 & 4 & 3 \\
Subsample &  & 0.9500±0.0124 & 0.1172±0.0229 & 0.9949±0.0025 & 150 & 4 & 3 \\
Bayesian RF & 1.0 & 0.9533±0.0102 & 0.1157±0.0245 & 0.9951±0.0021 & 150 & 4 & 3 \\
DM & 0.0481 & 0.9633±0.0126 & 0.1509±0.0128 & 0.9968±0.0015 & 150 & 4 & 3 \\
DM & 0.1661 & 0.9600±0.0130 & 0.1216±0.0166 & 0.9956±0.0022 & 150 & 4 & 3 \\
DM & 1.5547 & 0.9500±0.0134 & 0.1206±0.0223 & 0.9949±0.0025 & 150 & 4 & 3 \\
DM & 1.5865 & 0.9533±0.0142 & 0.1167±0.0210 & 0.9943±0.0026 & 150 & 4 & 3 \\
DM & 3.5442 & 0.9633±0.0105 & 0.1096±0.0203 & 0.9958±0.0018 & 150 & 4 & 3 \\
DM & 4.0289 & 0.9467±0.0133 & 0.1203±0.0238 & 0.9950±0.0024 & 150 & 4 & 3 \\
DM & 5.6992 & 0.9533±0.0133 & 0.1146±0.0225 & 0.9953±0.0021 & 150 & 4 & 3 \\
DM & 7.0225 & 0.9567±0.0112 & 0.1114±0.0221 & 0.9954±0.0023 & 150 & 4 & 3 \\
DM & 17.7015 & 0.9567±0.0100 & 0.1155±0.0206 & 0.9943±0.0027 & 150 & 4 & 3 \\
DM & 63.2004 & 0.9600±0.0097 & 0.1182±0.0227 & 0.9952±0.0023 & 150 & 4 & 3 \\
DW & 0.0481 & 0.9467±0.0102 & 0.1082±0.0198 & 0.9952±0.0021 & 150 & 4 & 3 \\
DW & 0.1661 & 0.9500±0.0102 & 0.1147±0.0230 & 0.9943±0.0019 & 150 & 4 & 3 \\
DW & 1.5547 & 0.9500±0.0090 & 0.1221±0.0271 & 0.9955±0.0019 & 150 & 4 & 3 \\
DW & 1.5865 & 0.9500±0.0090 & 0.1174±0.0222 & 0.9947±0.0022 & 150 & 4 & 3 \\
DW & 3.5442 & 0.9433±0.0087 & 0.1229±0.0263 & 0.9945±0.0021 & 150 & 4 & 3 \\
DW & 4.0289 & 0.9533±0.0089 & 0.2264±0.1243 & 0.9925±0.0039 & 150 & 4 & 3 \\
DW & 5.6992 & 0.9433±0.0087 & 0.1252±0.0267 & 0.9938±0.0021 & 150 & 4 & 3 \\
DW & 7.0225 & 0.9500±0.0090 & 0.2299±0.1231 & 0.9925±0.0035 & 150 & 4 & 3 \\
DW & 17.7015 & 0.9400±0.0067 & 0.1274±0.0247 & 0.9937±0.0021 & 150 & 4 & 3 \\
DW & 63.2004 & 0.9467±0.0089 & 0.3388±0.1608 & 0.9898±0.0046 & 150 & 4 & 3 \\

\end{tabular}
\end{center}
\end{table}

\begin{table}[h]
\caption{Results for Letter} \label{tab:letter}
\begin{center}
\begin{tabular}{llllllll}
\textbf{method} &\textbf{alpha} &\textbf{accuracy} &\textbf{logloss} &\textbf{auroc} &\textbf{n} &\textbf{d} &\textbf{C} \\
\hline \\
ET &  & 0.9723±0.0008 & 0.2714±0.0014 & 0.9997±0.0000 & 20000 & 16 & 26 \\
RF-bootstrap &  & 0.9644±0.0009 & 0.2761±0.0021 & 0.9996±0.0000 & 20000 & 16 & 26 \\
RF-no-bootstrap &  & 0.9679±0.0009 & 0.2272±0.0016 & 0.9996±0.0000 & 20000 & 16 & 26 \\
Subsample &  & 0.9644±0.0009 & 0.2733±0.0021 & 0.9995±0.0000 & 20000 & 16 & 26 \\
Bayesian RF & 1.0 & 0.9673±0.0008 & 0.2358±0.0016 & 0.9996±0.0000 & 20000 & 16 & 26 \\
DM & 0.0013 & 0.7991±0.0022 & 1.4427±0.0036 & 0.9814±0.0003 & 20000 & 16 & 26 \\
DM & 0.0017 & 0.8125±0.0020 & 1.3492±0.0035 & 0.9842±0.0002 & 20000 & 16 & 26 \\
DM & 0.0081 & 0.8856±0.0017 & 0.8726±0.0026 & 0.9950±0.0001 & 20000 & 16 & 26 \\
DM & 0.0083 & 0.8866±0.0019 & 0.8613±0.0023 & 0.9953±0.0001 & 20000 & 16 & 26 \\
DM & 0.0262 & 0.9238±0.0015 & 0.6248±0.0024 & 0.9980±0.0000 & 20000 & 16 & 26 \\
DM & 0.0360 & 0.9309±0.0014 & 0.5724±0.0026 & 0.9983±0.0000 & 20000 & 16 & 26 \\
DM & 0.2139 & 0.9548±0.0009 & 0.3785±0.0020 & 0.9993±0.0000 & 20000 & 16 & 26 \\
DM & 0.3448 & 0.9579±0.0011 & 0.3477±0.0019 & 0.9994±0.0000 & 20000 & 16 & 26 \\
DM & 0.4136 & 0.9584±0.0007 & 0.3388±0.0019 & 0.9994±0.0000 & 20000 & 16 & 26 \\
DM & 0.4779 & 0.9602±0.0012 & 0.3309±0.0022 & 0.9995±0.0000 & 20000 & 16 & 26 \\
DW & 0.0013 & 0.9149±0.0015 & 0.8356±0.0025 & 0.9968±0.0000 & 20000 & 16 & 26 \\
DW & 0.0017 & 0.9245±0.0012 & 0.7650±0.0025 & 0.9975±0.0000 & 20000 & 16 & 26 \\
DW & 0.0081 & 0.9572±0.0010 & 0.4452±0.0019 & 0.9993±0.0000 & 20000 & 16 & 26 \\
DW & 0.0083 & 0.9577±0.0008 & 0.4411±0.0019 & 0.9993±0.0000 & 20000 & 16 & 26 \\
DW & 0.0262 & 0.9654±0.0008 & 0.3193±0.0016 & 0.9996±0.0000 & 20000 & 16 & 26 \\
DW & 0.0360 & 0.9665±0.0007 & 0.2990±0.0015 & 0.9996±0.0000 & 20000 & 16 & 26 \\
DW & 0.2139 & 0.9682±0.0009 & 0.2512±0.0019 & 0.9996±0.0000 & 20000 & 16 & 26 \\
DW & 0.3448 & 0.9678±0.0008 & 0.2471±0.0018 & 0.9996±0.0000 & 20000 & 16 & 26 \\
DW & 0.4136 & 0.9674±0.0007 & 0.2431±0.0017 & 0.9996±0.0000 & 20000 & 16 & 26 \\
DW & 0.4779 & 0.9674±0.0010 & 0.2436±0.0020 & 0.9996±0.0000 & 20000 & 16 & 26 \\

\end{tabular}
\end{center}
\end{table}

\begin{table}[h]
\caption{Results for Magic04} \label{tab:magic04}
\begin{center}
\begin{tabular}{llllllll}
\textbf{method} &\textbf{alpha} &\textbf{accuracy} &\textbf{logloss} &\textbf{auroc} &\textbf{n} &\textbf{d} &\textbf{C} \\
\hline \\
ET &  & 0.8807±0.0020 & 0.2991±0.0026 & 0.9398±0.0013 & 19020 & 10 & 2 \\
RF-bootstrap &  & 0.8835±0.0016 & 0.2914±0.0020 & 0.9402±0.0012 & 19020 & 10 & 2 \\
RF-no-bootstrap &  & 0.8831±0.0015 & 0.2932±0.0033 & 0.9393±0.0012 & 19020 & 10 & 2 \\
Subsample &  & 0.8835±0.0018 & 0.2935±0.0024 & 0.9404±0.0012 & 19020 & 10 & 2 \\
Bayesian RF & 1.0 & 0.8836±0.0014 & 0.2914±0.0023 & 0.9406±0.0011 & 19020 & 10 & 2 \\
DM & 0 & 0.6806±0.0039 & 0.5723±0.0022 & 0.8238±0.0064 & 19020 & 10 & 2 \\
DM & 0.0001 & 0.7509±0.0045 & 0.5129±0.0035 & 0.8680±0.0034 & 19020 & 10 & 2 \\
DM & 0.0001 & 0.7510±0.0025 & 0.5080±0.0016 & 0.8720±0.0027 & 19020 & 10 & 2 \\
DM & 0.0003 & 0.8049±0.0033 & 0.4574±0.0013 & 0.8963±0.0016 & 19020 & 10 & 2 \\
DM & 0.0008 & 0.8323±0.0017 & 0.4171±0.0014 & 0.9093±0.0018 & 19020 & 10 & 2 \\
DM & 0.0045 & 0.8572±0.0014 & 0.3644±0.0015 & 0.9226±0.0016 & 19020 & 10 & 2 \\
DM & 0.0611 & 0.8761±0.0019 & 0.3158±0.0020 & 0.9354±0.0013 & 19020 & 10 & 2 \\
DM & 0.2551 & 0.8816±0.0018 & 0.3013±0.0020 & 0.9390±0.0012 & 19020 & 10 & 2 \\
DM & 0.3127 & 0.8818±0.0019 & 0.2990±0.0018 & 0.9393±0.0012 & 19020 & 10 & 2 \\
DM & 0.6278 & 0.8838±0.0015 & 0.2970±0.0027 & 0.9401±0.0011 & 19020 & 10 & 2 \\
DW & 0 & 0.7787±0.0042 & 0.4951±0.0017 & 0.8786±0.0019 & 19020 & 10 & 2 \\
DW & 0.0001 & 0.8319±0.0026 & 0.4337±0.0012 & 0.9036±0.0009 & 19020 & 10 & 2 \\
DW & 0.0001 & 0.8329±0.0021 & 0.4296±0.0017 & 0.9063±0.0020 & 19020 & 10 & 2 \\
DW & 0.0003 & 0.8507±0.0018 & 0.3927±0.0012 & 0.9176±0.0016 & 19020 & 10 & 2 \\
DW & 0.0008 & 0.8608±0.0012 & 0.3634±0.0012 & 0.9251±0.0015 & 19020 & 10 & 2 \\
DW & 0.0045 & 0.8764±0.0016 & 0.3241±0.0017 & 0.9352±0.0013 & 19020 & 10 & 2 \\
DW & 0.0611 & 0.8848±0.0018 & 0.2982±0.0015 & 0.9410±0.0012 & 19020 & 10 & 2 \\
DW & 0.2551 & 0.8841±0.0015 & 0.2916±0.0019 & 0.9414±0.0012 & 19020 & 10 & 2 \\
DW & 0.3127 & 0.8849±0.0015 & 0.2928±0.0026 & 0.9411±0.0012 & 19020 & 10 & 2 \\
DW & 0.6278 & 0.8835±0.0016 & 0.2939±0.0028 & 0.9410±0.0012 & 19020 & 10 & 2 \\

\end{tabular}
\end{center}
\end{table}

\begin{table}[h]
\caption{Results for Optdigits} \label{tab:optdigits}
\begin{center}
\begin{tabular}{llllllll}
\textbf{method} &\textbf{alpha} &\textbf{accuracy} &\textbf{logloss} &\textbf{auroc} &\textbf{n} &\textbf{d} &\textbf{C} \\
\hline \\
ET &  & 0.9861±0.0010 & 0.2065±0.0011 & 0.9997±0.0000 & 5620 & 64 & 10 \\
RF-bootstrap &  & 0.9844±0.0010 & 0.2159±0.0018 & 0.9997±0.0000 & 5620 & 64 & 10 \\
RF-no-bootstrap &  & 0.9861±0.0010 & 0.1822±0.0015 & 0.9997±0.0000 & 5620 & 64 & 10 \\
Subsample &  & 0.9846±0.0011 & 0.2129±0.0015 & 0.9996±0.0000 & 5620 & 64 & 10 \\
Bayesian RF & 1.0 & 0.9843±0.0013 & 0.1902±0.0014 & 0.9997±0.0000 & 5620 & 64 & 10 \\
DM & 0 & 0.4194±0.0243 & 2.0544±0.0049 & 0.9030±0.0028 & 5620 & 64 & 10 \\
DM & 0.0001 & 0.7077±0.0115 & 1.8090±0.0063 & 0.9569±0.0018 & 5620 & 64 & 10 \\
DM & 0.0011 & 0.9162±0.0022 & 1.0917±0.0035 & 0.9922±0.0003 & 5620 & 64 & 10 \\
DM & 0.0013 & 0.9252±0.0015 & 1.0453±0.0033 & 0.9923±0.0004 & 5620 & 64 & 10 \\
DM & 0.0026 & 0.9399±0.0024 & 0.8644±0.0017 & 0.9953±0.0003 & 5620 & 64 & 10 \\
DM & 0.0226 & 0.9671±0.0021 & 0.4641±0.0025 & 0.9988±0.0001 & 5620 & 64 & 10 \\
DM & 0.1261 & 0.9789±0.0012 & 0.3089±0.0015 & 0.9994±0.0001 & 5620 & 64 & 10 \\
DM & 0.1583 & 0.9784±0.0014 & 0.2957±0.0017 & 0.9995±0.0000 & 5620 & 64 & 10 \\
DM & 0.2982 & 0.9817±0.0014 & 0.2678±0.0019 & 0.9995±0.0001 & 5620 & 64 & 10 \\
DM & 0.6041 & 0.9825±0.0014 & 0.2452±0.0020 & 0.9996±0.0000 & 5620 & 64 & 10 \\
DW & 0 & 0.8270±0.0102 & 1.5911±0.0062 & 0.9755±0.0011 & 5620 & 64 & 10 \\
DW & 0.0001 & 0.8976±0.0037 & 1.2571±0.0049 & 0.9882±0.0004 & 5620 & 64 & 10 \\
DW & 0.0011 & 0.9590±0.0021 & 0.6468±0.0023 & 0.9979±0.0001 & 5620 & 64 & 10 \\
DW & 0.0013 & 0.9608±0.0017 & 0.6156±0.0024 & 0.9980±0.0001 & 5620 & 64 & 10 \\
DW & 0.0026 & 0.9676±0.0015 & 0.4959±0.0018 & 0.9987±0.0001 & 5620 & 64 & 10 \\
DW & 0.0226 & 0.9822±0.0013 & 0.2630±0.0013 & 0.9996±0.0000 & 5620 & 64 & 10 \\
DW & 0.1261 & 0.9853±0.0009 & 0.2081±0.0015 & 0.9997±0.0000 & 5620 & 64 & 10 \\
DW & 0.1583 & 0.9852±0.0012 & 0.2068±0.0014 & 0.9997±0.0000 & 5620 & 64 & 10 \\
DW & 0.2982 & 0.9854±0.0011 & 0.1993±0.0014 & 0.9997±0.0000 & 5620 & 64 & 10 \\
DW & 0.6041 & 0.9859±0.0011 & 0.1937±0.0013 & 0.9997±0.0000 & 5620 & 64 & 10 \\

\end{tabular}
\end{center}
\end{table}

\begin{table}[h]
\caption{Results for Page Blocks} \label{tab:page-blocks}
\begin{center}
\begin{tabular}{llllllll}
\textbf{method} &\textbf{alpha} &\textbf{accuracy} &\textbf{logloss} &\textbf{auroc} &\textbf{n} &\textbf{d} &\textbf{C} \\
\hline \\
ET &  & 0.9709±0.0012 & 0.2461±0.0134 & 0.9777±0.0014 & 5473 & 10 & 5 \\
RF-bootstrap &  & 0.9728±0.0012 & 0.1610±0.0140 & 0.9881±0.0010 & 5473 & 10 & 5 \\
RF-no-bootstrap &  & 0.9692±0.0014 & 0.2775±0.0183 & 0.9757±0.0021 & 5473 & 10 & 5 \\
Subsample &  & 0.9722±0.0011 & 0.1648±0.0151 & 0.9888±0.0012 & 5473 & 10 & 5 \\
Bayesian RF & 1.0 & 0.9701±0.0015 & 0.2616±0.0152 & 0.9773±0.0016 & 5473 & 10 & 5 \\
DM & 0 & 0.8977±0.0000 & 0.4578±0.0228 & 0.8478±0.0149 & 5473 & 10 & 5 \\
DM & 0.0006 & 0.9065±0.0017 & 0.2171±0.0032 & 0.9815±0.0014 & 5473 & 10 & 5 \\
DM & 0.0033 & 0.9493±0.0017 & 0.1568±0.0050 & 0.9888±0.0011 & 5473 & 10 & 5 \\
DM & 0.0058 & 0.9544±0.0015 & 0.1433±0.0070 & 0.9903±0.0012 & 5473 & 10 & 5 \\
DM & 0.0063 & 0.9564±0.0015 & 0.1409±0.0063 & 0.9902±0.0011 & 5473 & 10 & 5 \\
DM & 0.0075 & 0.9587±0.0014 & 0.1305±0.0046 & 0.9918±0.0006 & 5473 & 10 & 5 \\
DM & 0.0239 & 0.9647±0.0016 & 0.1140±0.0050 & 0.9925±0.0009 & 5473 & 10 & 5 \\
DM & 0.0362 & 0.9674±0.0015 & 0.1117±0.0068 & 0.9920±0.0010 & 5473 & 10 & 5 \\
DM & 0.0551 & 0.9699±0.0015 & 0.1080±0.0055 & 0.9927±0.0008 & 5473 & 10 & 5 \\
DM & 0.0584 & 0.9693±0.0011 & 0.1129±0.0080 & 0.9918±0.0010 & 5473 & 10 & 5 \\
DW & 0 & 0.8977±0.0000 & 0.3002±0.0052 & 0.9566±0.0023 & 5473 & 10 & 5 \\
DW & 0.0006 & 0.9528±0.0013 & 0.1514±0.0066 & 0.9887±0.0013 & 5473 & 10 & 5 \\
DW & 0.0033 & 0.9677±0.0015 & 0.1075±0.0060 & 0.9919±0.0012 & 5473 & 10 & 5 \\
DW & 0.0058 & 0.9704±0.0012 & 0.1070±0.0051 & 0.9914±0.0010 & 5473 & 10 & 5 \\
DW & 0.0063 & 0.9708±0.0011 & 0.1052±0.0058 & 0.9906±0.0016 & 5473 & 10 & 5 \\
DW & 0.0075 & 0.9721±0.0012 & 0.1205±0.0064 & 0.9883±0.0018 & 5473 & 10 & 5 \\
DW & 0.0239 & 0.9733±0.0013 & 0.1543±0.0172 & 0.9860±0.0026 & 5473 & 10 & 5 \\
DW & 0.0362 & 0.9726±0.0013 & 0.1692±0.0132 & 0.9824±0.0019 & 5473 & 10 & 5 \\
DW & 0.0551 & 0.9716±0.0013 & 0.1782±0.0168 & 0.9811±0.0021 & 5473 & 10 & 5 \\
DW & 0.0584 & 0.9722±0.0014 & 0.1757±0.0171 & 0.9819±0.0026 & 5473 & 10 & 5 \\
\end{tabular}
\end{center}
\end{table}

\begin{table}[h]
\caption{Results for Pendigits} \label{tab:pendigits}
\begin{center}
\begin{tabular}{llllllll}
\textbf{method} &\textbf{alpha} &\textbf{accuracy} &\textbf{logloss} &\textbf{auroc} &\textbf{n} &\textbf{d} &\textbf{C} \\
\hline \\
ET &  & 0.9935±0.0003 & 0.0796±0.0039 & 0.9998±0.0001 & 10992 & 16 & 10 \\
RF-bootstrap &  & 0.9918±0.0005 & 0.0874±0.0044 & 0.9997±0.0001 & 10992 & 16 & 10 \\
RF-no-bootstrap &  & 0.9928±0.0005 & 0.0730±0.0046 & 0.9997±0.0001 & 10992 & 16 & 10 \\
Subsample &  & 0.9918±0.0004 & 0.0885±0.0049 & 0.9997±0.0001 & 10992 & 16 & 10 \\
Bayesian RF & 1.0 & 0.9933±0.0004 & 0.0775±0.0049 & 0.9997±0.0001 & 10992 & 16 & 10 \\

DM & 0.0001 & 0.7895±0.0097 & 1.3476±0.0078 & 0.9758±0.0009 & 10992 & 16 & 10 \\
DM & 0.0005 & 0.8851±0.0026 & 0.8485±0.0023 & 0.9916±0.0002 & 10992 & 16 & 10 \\
DM & 0.0005 & 0.8789±0.0036 & 0.8388±0.0025 & 0.9912±0.0003 & 10992 & 16 & 10 \\
DM & 0.0033 & 0.9495±0.0012 & 0.4226±0.0039 & 0.9981±0.0001 & 10992 & 16 & 10 \\
DM & 0.1355 & 0.9883±0.0004 & 0.1271±0.0041 & 0.9997±0.0001 & 10992 & 16 & 10 \\
DM & 0.1541 & 0.9885±0.0005 & 0.1249±0.0041 & 0.9997±0.0001 & 10992 & 16 & 10 \\
DM & 0.1792 & 0.9890±0.0006 & 0.1255±0.0050 & 0.9996±0.0001 & 10992 & 16 & 10 \\
DM & 0.2325 & 0.9894±0.0004 & 0.1167±0.0048 & 0.9997±0.0001 & 10992 & 16 & 10 \\
DM & 0.4928 & 0.9904±0.0006 & 0.1023±0.0041 & 0.9997±0.0001 & 10992 & 16 & 10 \\
DM & 0.5211 & 0.9904±0.0005 & 0.1035±0.0038 & 0.9997±0.0001 & 10992 & 16 & 10 \\
DW & 0.0001 & 0.8994±0.0023 & 0.8102±0.0049 & 0.9925±0.0003 & 10992 & 16 & 10 \\
DW & 0.0005 & 0.9521±0.0018 & 0.4468±0.0044 & 0.9982±0.0002 & 10992 & 16 & 10 \\
DW & 0.0005 & 0.9518±0.0015 & 0.4405±0.0035 & 0.9983±0.0001 & 10992 & 16 & 10 \\
DW & 0.0033 & 0.9835±0.0007 & 0.2065±0.0044 & 0.9995±0.0001 & 10992 & 16 & 10 \\
DW & 0.1355 & 0.9930±0.0005 & 0.0774±0.0040 & 0.9997±0.0001 & 10992 & 16 & 10 \\
DW & 0.1541 & 0.9934±0.0005 & 0.0784±0.0045 & 0.9997±0.0001 & 10992 & 16 & 10 \\
DW & 0.1792 & 0.9933±0.0004 & 0.0794±0.0054 & 0.9997±0.0001 & 10992 & 16 & 10 \\
DW & 0.2325 & 0.9934±0.0005 & 0.0786±0.0049 & 0.9997±0.0001 & 10992 & 16 & 10 \\
DW & 0.4928 & 0.9934±0.0003 & 0.0789±0.0057 & 0.9997±0.0001 & 10992 & 16 & 10 \\
DW & 0.5211 & 0.9929±0.0005 & 0.0757±0.0046 & 0.9997±0.0001 & 10992 & 16 & 10 \\
\end{tabular}
\end{center}
\end{table}

\begin{table}[h]
\caption{Results for Phoneme} \label{tab:phoneme}
\begin{center}
\begin{tabular}{llllllll}
\textbf{method} &\textbf{alpha} &\textbf{accuracy} &\textbf{logloss} &\textbf{auroc} &\textbf{n} &\textbf{d} &\textbf{C} \\
\hline \\
ET &  & 0.9182±0.0022 & 0.2190±0.0075 & 0.9702±0.0016 & 5404 & 5 & 2 \\
RF-bootstrap &  & 0.9127±0.0022 & 0.2588±0.0128 & 0.9634±0.0018 & 5404 & 5 & 2 \\
RF-no-bootstrap &  & 0.9131±0.0021 & 0.2452±0.0118 & 0.9668±0.0016 & 5404 & 5 & 2 \\
Subsample &  & 0.9119±0.0026 & 0.2531±0.0104 & 0.9631±0.0019 & 5404 & 5 & 2 \\
Bayesian RF & 1.0 & 0.9169±0.0022 & 0.2403±0.0116 & 0.9676±0.0016 & 5404 & 5 & 2 \\

DM & 0.0002 & 0.7340±0.0030 & 0.4718±0.0044 & 0.8364±0.0041 & 5404 & 5 & 2 \\
DM & 0.0017 & 0.8019±0.0038 & 0.4024±0.0027 & 0.8814±0.0026 & 5404 & 5 & 2 \\
DM & 0.0029 & 0.8188±0.0040 & 0.3877±0.0032 & 0.8931±0.0031 & 5404 & 5 & 2 \\
DM & 0.0154 & 0.8544±0.0036 & 0.3374±0.0035 & 0.9246±0.0029 & 5404 & 5 & 2 \\
DM & 0.0430 & 0.8741±0.0030 & 0.3112±0.0050 & 0.9385±0.0024 & 5404 & 5 & 2 \\
DM & 0.0448 & 0.8745±0.0024 & 0.3074±0.0030 & 0.9387±0.0021 & 5404 & 5 & 2 \\
DM & 0.0841 & 0.8836±0.0038 & 0.2924±0.0049 & 0.9461±0.0022 & 5404 & 5 & 2 \\
DM & 0.1244 & 0.8872±0.0030 & 0.2813±0.0043 & 0.9502±0.0021 & 5404 & 5 & 2 \\
DM & 0.1983 & 0.8935±0.0023 & 0.2735±0.0055 & 0.9535±0.0019 & 5404 & 5 & 2 \\
DM & 0.4224 & 0.9014±0.0022 & 0.2704±0.0108 & 0.9578±0.0020 & 5404 & 5 & 2 \\
DW & 0.0002 & 0.7985±0.0039 & 0.4096±0.0027 & 0.8777±0.0023 & 5404 & 5 & 2 \\
DW & 0.0017 & 0.8552±0.0025 & 0.3391±0.0029 & 0.9228±0.0023 & 5404 & 5 & 2 \\
DW & 0.0029 & 0.8669±0.0022 & 0.3180±0.0028 & 0.9344±0.0019 & 5404 & 5 & 2 \\
DW & 0.0154 & 0.9052±0.0026 & 0.2704±0.0087 & 0.9590±0.0022 & 5404 & 5 & 2 \\
DW & 0.0430 & 0.9149±0.0024 & 0.2425±0.0100 & 0.9665±0.0019 & 5404 & 5 & 2 \\
DW & 0.0448 & 0.9148±0.0025 & 0.2420±0.0085 & 0.9665±0.0017 & 5404 & 5 & 2 \\
DW & 0.0841 & 0.9191±0.0024 & 0.2528±0.0118 & 0.9678±0.0018 & 5404 & 5 & 2 \\
DW & 0.1244 & 0.9185±0.0017 & 0.2365±0.0095 & 0.9681±0.0017 & 5404 & 5 & 2 \\
DW & 0.1983 & 0.9179±0.0023 & 0.2351±0.0113 & 0.9681±0.0017 & 5404 & 5 & 2 \\
DW & 0.4224 & 0.9168±0.0021 & 0.2393±0.0145 & 0.9678±0.0018 & 5404 & 5 & 2 \\
\end{tabular}
\end{center}
\end{table}

\end{document}